\documentclass[sigconf]{aamas} 

\usepackage{balance} % for balancing columns on the final page
\usepackage{flushend}

\usepackage{newfloat}
\usepackage{listings}
\usepackage{amsthm}
\usepackage{amsmath}
\usepackage{tabularx}
\usepackage{color}
\usepackage{subfigure}
\usepackage[linesnumbered,ruled,vlined]{algorithm2e}
\usepackage{breqn}

\setcopyright{none}
\acmConference[AAMAS '22]{Proc.\@ of the 21st International Conference
on Autonomous Agents and Multiagent Systems (AAMAS 2022)}{May 9--13, 2022}
{Online}{ P. Faliszewski, V. Mascardi, C. Pelachaud, M.E. Taylor (eds.)}
\copyrightyear{2022}
\acmYear{2022}
\acmDOI{}
\acmPrice{}
\acmISBN{}

\setcopyright{none}
\renewcommand\footnotetextcopyrightpermission[1]{}
\newcommand\blfootnote[1]{%
  \begingroup
  \renewcommand\thefootnote{}\footnote{#1}%
  \addtocounter{footnote}{-1}%
  \endgroup
}

\acmSubmissionID{669}

\title{Using Deep Learning to Bootstrap Abstractions for \\  Hierarchical Robot Planning \\ (Extended Version)}
\author{Naman Shah}
\affiliation{
  \institution{Arizona State University}
  \city{Tempe}
  \country{USA}}
\email{namanshah@asu.edu}

\author{Siddharth Srivastava}
\affiliation{
  \institution{Arizona State University}
  \city{Tempe}
  \country{USA}}
\email{siddharths@asu.edu}

\begin{abstract}

  This paper addresses the problem of learning abstractions that
  boost robot planning performance while providing strong
  guarantees of reliability.  Although state-of-the-art
  hierarchical robot planning algorithms allow robots to
  efficiently compute long-horizon motion plans for achieving user
  desired tasks, these methods typically rely upon environment-dependent state and action
  abstractions that need to be hand-designed by experts.

    We present a new approach for bootstrapping the entire
  hierarchical planning process. This allows us to compute abstract states and
  actions for new environments automatically using
  the critical regions predicted by a deep neural network with an
  auto-generated robot-specific architecture. We show that the learned
  abstractions can be used with a novel multi-source bi-directional hierarchical
  robot planning algorithm that is sound and probabilistically
  complete. An extensive empirical evaluation on twenty different
  settings using holonomic and non-holonomic robots shows that (a) our learned abstractions
  provide the information necessary for efficient multi-source hierarchical
  planning; and that (b) this approach of  learning, abstractions, and
  planning outperforms state-of-the-art baselines by nearly a
  factor of ten in terms of planning time on test environments not
  seen during training.

\end{abstract}


\newcommand{\BibTeX}{\rm B\kern-.05em{\sc i\kern-.025em b}\kern-.08em\TeX}
\newtheorem{theorem}{Theorem}[section]

\newtheorem{definition}{Definition}
\newcommand\X{{\mathcal{X}}}

\newcommand\Xf{{\mathcal{X}_{\text{free}}}}
\newcommand\Xo{{\mathcal{X}_{\text{obs}}}}
\newcommand\Sm{{\mathcal{S}}}

\usepackage{lineno}
\begin{document}

%%% The following commands remove the headers in your paper. For final 
%%% papers, these will be inserted during the pagination process.

\pagestyle{fancy}
\fancyhead{}

\cfoot{\thepage}

%%% The next command prints the information defined in the preamble.

\maketitle 

%%%%%%%%%%%%%%%%%%%%%%%%%%%%%%%%%%%%%%%%%%%%%%%%%%%%%%%%%%%%%%%%%%%%%%%%
\section{Introduction}
\label{sec:intro}
\blfootnote{This paper is originally published in proc. of $21^{st}$ International Conference on Autonomous Agents and Multiagent Systems (AAMAS, 2022). Please cite the original work as: \\ Naman Shah and Siddharth Srivastava. 2022. Using Deep Learning to Boot-
strap Abstractions for Hierarchical Robot Planning. In
Proc. of the 21st International Conference on Autonomous Agents and Multi-agent Systems (AAMAS 2022), Online, May 9-13, 2022, IFAAMAS, 9 pages. }

Autonomous robots need to be able to efficiently compute long-horizon motion plans for achieving user desired tasks~\citep{lo2018petlon,duckworth2016unsupervised,luperto2020exploration}. E.g., consider a scenario where a robot $R$ in a household environment is tasked to reach kitchen $K$ from its current location $B1$ (Fig.~\ref{fig:example1}). State-of-the-art motion planning algorithms such as PRM~\citep{kavraki1996probabilistic} and RRT~\citep{lavalle1998rapidly} use random sampling of low-level configurations to compute a path from the robot's current location $B1$ to its target location $K$. Such sampling-based methods fail to efficiently sample configurations from confined spaces such as doorways and corridors under uniform sampling ~\citep{dan_llp}.

On the other hand, humans tend to reason using abstract, high-level actions. E.g., in the same scenario, we would use high-level (abstract) actions such as  \emph{``go out of the room B1''}, \emph{``pass through the corridor''}, and \emph{``enter the kitchen K''}. These abstract actions allow us to reason over a long-horizon easily.
Currently, domain experts need to create such abstractions by hand. This limits the scope and scalability of approaches for hierarchical planning (e.g.,~\citep{garrett2020pddlstream,shah2020anytime,dantam2018incremental}) to situations and domains where experts are available and able to correctly intuit the required abstractions.

This paper shows that the required abstractions can be learned by identifying regions in the environment that are important to solve motion planning problems (loosely similar to landmarks in task planning).~\citet{dan_llp} define such regions as \emph{critical regions}. Critical regions are analogous to \emph{landmarks} in automated planning but unlike landmarks, critical regions do not necessarily have to be reached to achieve the goal. Fig. \ref{fig:demo_reg_1} shows a few candidate critical regions for the environment in Fig. \ref{fig:demo_env_1}.

\begin{figure}
    \centering

    \subfigure[]{\includegraphics[width= 0.3\columnwidth]{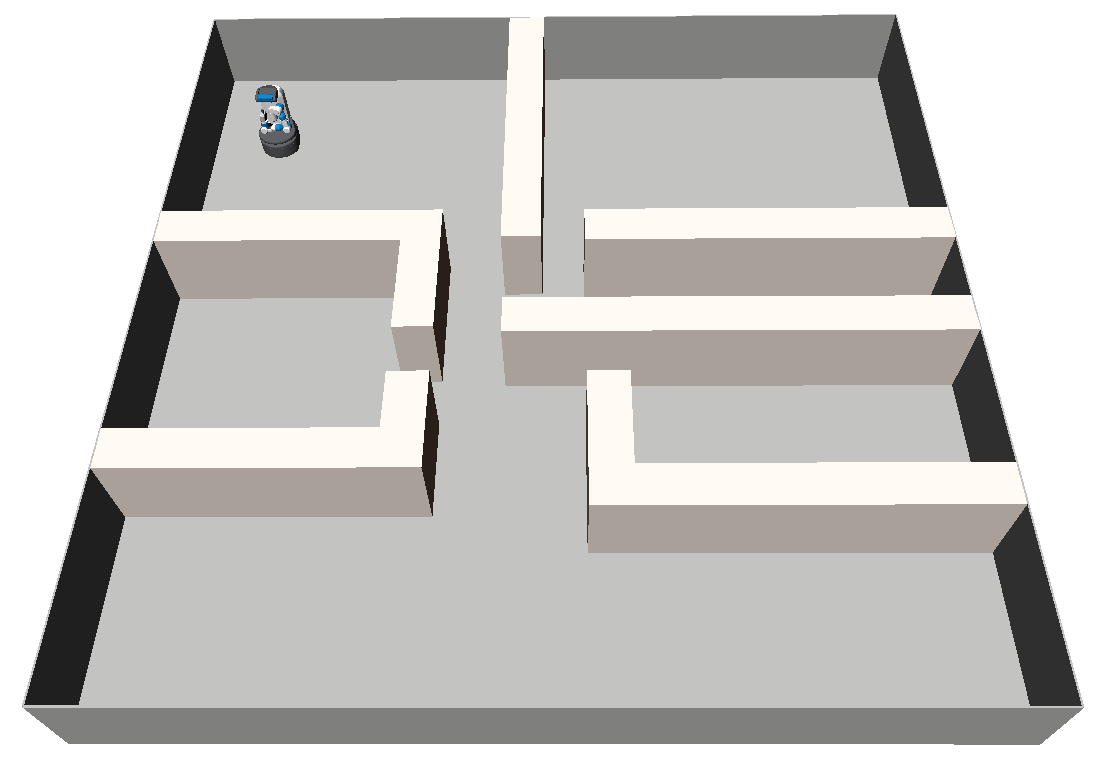}
    \label{fig:demo_env_1}}
    \subfigure[]{\includegraphics[width= 0.3\columnwidth]{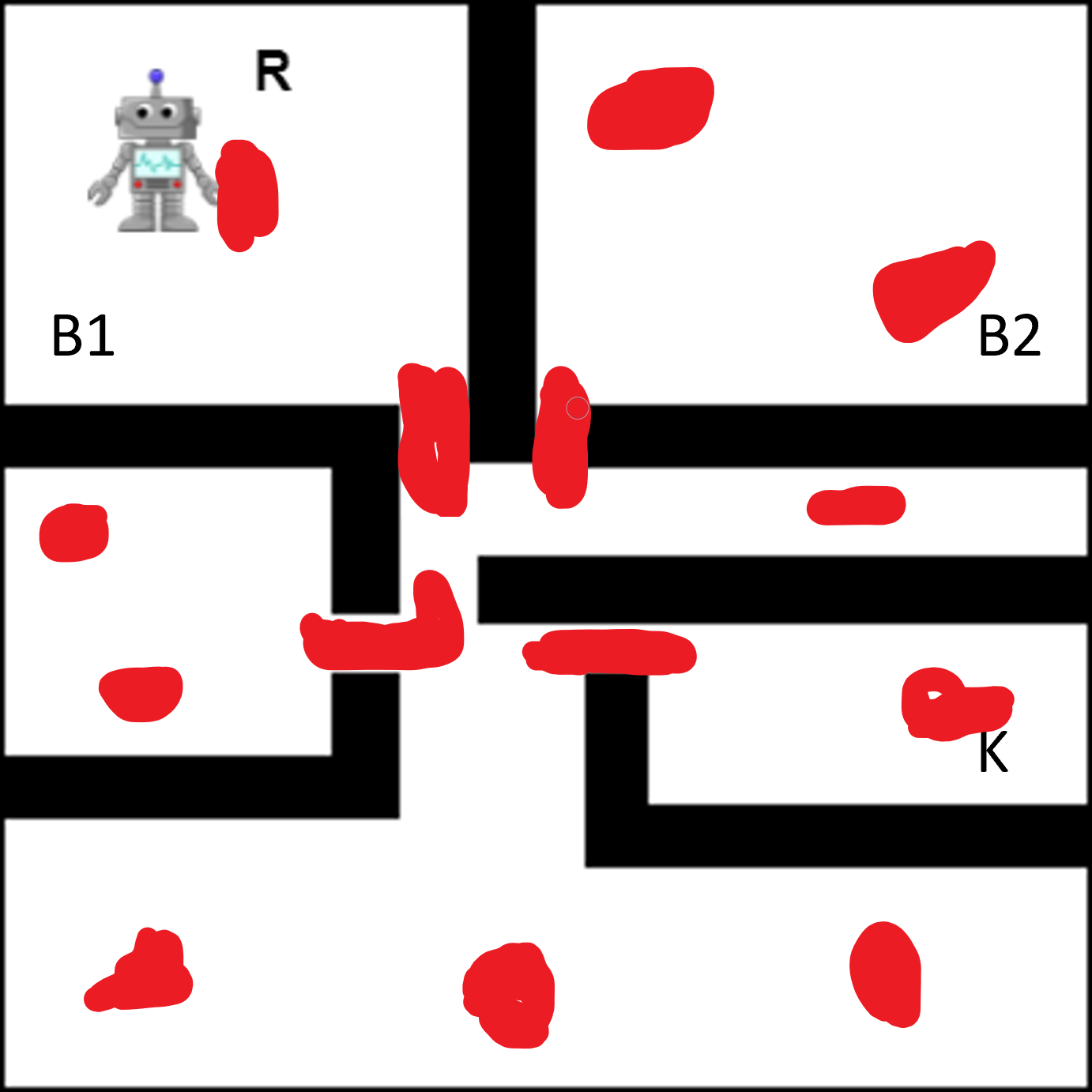}
    \label{fig:demo_reg_1}}
    \subfigure[]{\includegraphics[width= 0.3\columnwidth]{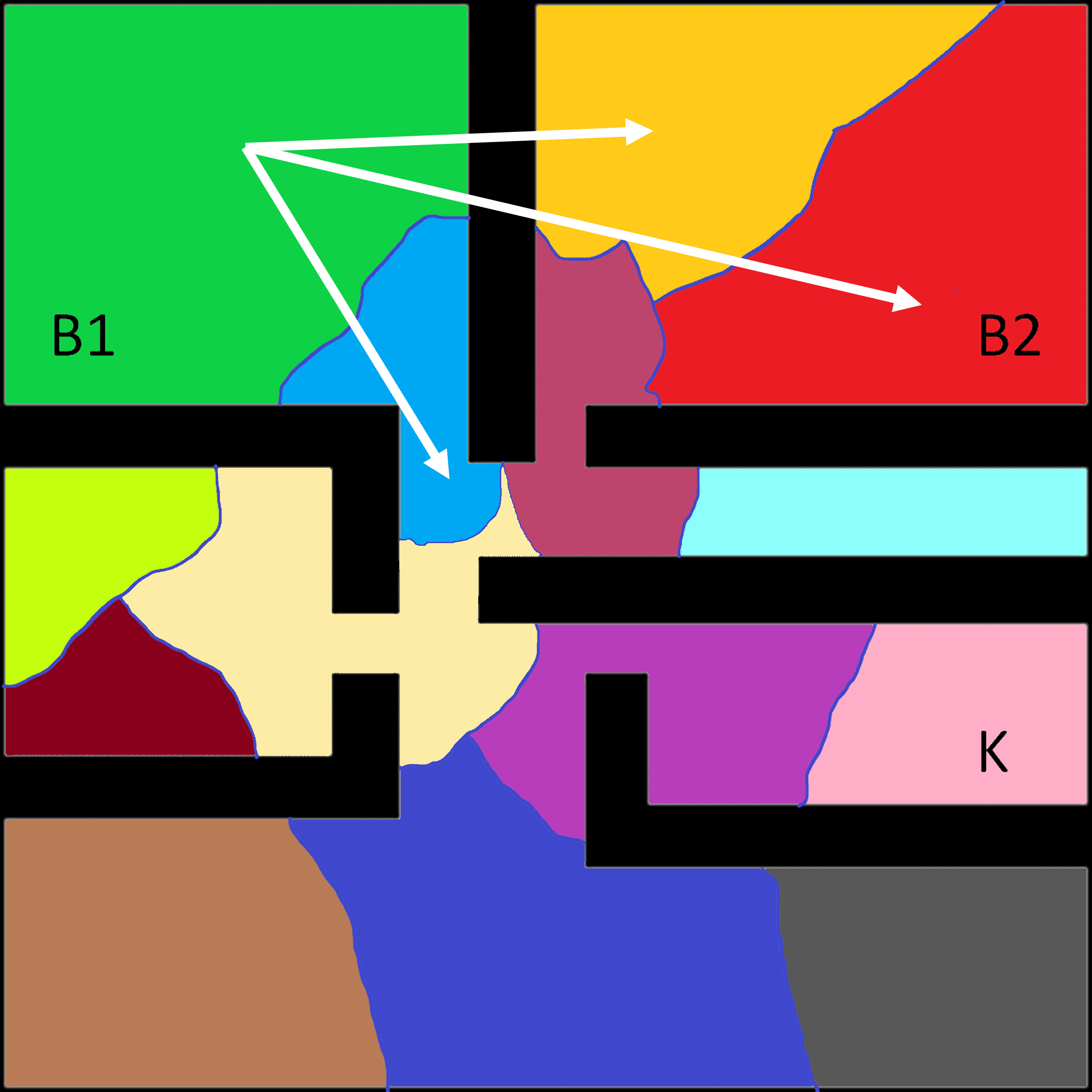}\label{fig:demo_voronoi_1}}

    \caption{(a) An illustrative environment for a motion planning problem. The robot (R) is tasked to reach the kitchen (K). Red regions (b) show predicted critical regions in the environment. Lastly, (c) shows a projection of the computed state abstraction. Each colored cell represents an abstract state. Arrows show examples of abstract actions, defined as transitions between abstract states. } 

    \Description{
       Illustrative figure for our appoach: There are three figures. The first figure shows an example of a household environemt with a fetch robot in the top left corner. The second image shows the example of critical regions highlighted in the image as red blobs. These are the critical regions that our approach aim to automatically identify for unseen environments. The third figure shows an example of state and action abstraction that our approach aims to construct using the critical regions in the second figure.
    }
    
    \label{fig:example1}
\end{figure}

In this paper, we investigate two major questions: $1)$ Can we use deep learning to automatically generate such abstractions for new environments? And $2)$ how can we use them to enable safe and more efficient planning algorithms? The main contributions of this paper are: a formal foundation for hierarchical (state and action) abstractions based on the critical regions (similar to Fig. \ref{fig:demo_voronoi_1}) and a novel algorithm for on-the-fly construction of such hierarchical abstraction using predicted critical regions. Abstractions computed using learning in this manner can be difficult to use. In fact, our preliminary experiments showed that they do not help in standard single-source goal-directed search due to collision with obstacles in the environment. 

An additional contribution of this work is the finding that these abstractions empower effective multi-source, multi-directional search algorithms. In general, such search algorithms can be difficult to use due to the absence of information about states that may lie close to a path to the goal. However, we found that when used with our auto-discovered abstract states, these algorithms significantly outperform existing baselines.

Our formal framework provides a way to generate sound abstractions that satisfy \emph{downward refinement property}~\citep{bacchus1991downward} for holonomic robots along with probabilistic completeness in the general case. Our exhaustive empirical evaluation shows that our approach outperforms state-of-the-art sampling and learning-based motion planners and requires significantly less time. This evaluation includes evaluation on a total of twenty different settings with four different robots, which include holonomic and non-holonomic robots.

The rest of the paper is organized as follows: Sec. \ref{sec:related} discusses some of the existing related approaches, Sec. \ref{sec:background} introduces some concepts required to understand our approach, Sec. \ref{sec:formal} defines the formal framework for our algorithm, Sec. \ref{sec:algo} presents our approach and theoretical results in detail. Finally, Sec. \ref{sec:evaluation} presents our extensive empirical evaluation. 

\section{Related Work}
\label{sec:related}
Much of the prior work on the topic is focused on decomposing a motion planning problem into smaller subproblems to reduce its complexity. Several approaches have been proposed that use state decomposition to reduce the complexity of a motion planning problem. \emph{Vertical cell decomposition}~\citep{chazelle1985approximation} partitions the state space into a collection of vertical cells and computes a roadmap that passes through all of these cells. \citet{kavraki_decomp} propose a hierarchical method that uses  
% method that decomposes the overall motion planning problem into multiple subproblems and stitches the solution for each of the subproblems to generate the complete motion plans. They use 
\emph{wavefront expansion} to compute the decomposition of the state space. While these approaches establish the foundation of decomposition-based motion planning, partitions generated through such approaches are arbitrary and do not provide any guarantees of completeness. \citet{simsek2005} use graph cut over local transition graph to identify interface points  between highly dense regions. They use these interface points to learn options that take the agent from one region to another region. One of the major distinctions between theirs and our approach is that our approach operates with continuous state and actions spaces, but their approach requires discrete actions. Their approach also requires collecting local experience for every new environment while our approach learns the model that identifies critical regions once and uses the same model for every new environment. Additionally, their approach is \emph{dual} to our approach as it aims to identify interface points between regions while our approach aims to directly identify these regions in the environment using a pre-trained network and  does not need to identify interface points.

~\citet{zhang2018learning} use rejection sampling to reject unrelated samples to speed up SBMPs.They use reinforcement learning to learn a policy that decides to accept or reject a new sample to expand the search tree. While their approach reduces the search space to compute the path, it still needs to process samples generated from regions that are irrelevant for the current problem. On the other hand, our hierarchical approach refines abstract plans into low-level motion plans which reduces the number of unnecessary samples. TogglePRM~\citep{denny-wafr-2012} maintains roadmaps for free space and obstacle space in the configuration space to estimate the narrow passages and sample points from these narrow passages. This approach works well for environments with $\alpha$-$\epsilon$-separable passages, even though it does not compute high-level abstractions.

Multiple approaches have used statistical learning to boost motion planning. ~\citet{wang2021survey} present a comprehensive survey of methods that utilize a variety of learning methods to improve the efficiency of SBMPs. Multiple approaches discussed by ~\citet{wang2021survey} use end-to-end deep learning to learn low-level reactive policies. End-to-end approaches are attractive given if they succeed, they can compute solutions much faster than traditional approaches, but it is  not exactly clear under which conditions these algorithms would succeed. Formally, these end-to-end deep learning-based approaches lack the guarantees about completeness and soundness that our approach provides. \citet{wang2021survey} also discuss approaches that use learning to aid sampling-based motion planning. We discuss a few of these approaches that are relevant to this work.  \citet{kurutach2018learning} uses \emph{InfoGAN}~\citep{chen2016infogan} to learn state-space partitioning for simple \emph{SE$^2$} robots. While their empirical evaluation shows promising results, similarly to previous decomposition-based approaches, they do not provide any proof of completeness. It is also not clear how their approach would scale to configuration spaces that had more than two dimensions. On the other hand, our approach provides formal guarantees of completeness and soundness (for holonomic robots) and scales to high-dimensional spaces.

\citet{ichter2018learning} and \citet{kumar2019} use a conditional variational autoencoder (CVAE) \citep{Sohn2015LearningSO} to learn  sampling distributions for the motion planning problems. \citet{ichter2020} use \emph{betweenness centrality} to learn criticality score for low-level configurations. They uniformly sample a set of configurations from the environment and use configurations with higher criticality from this set to generate a  roadmap. Their results show significant improvement over vanilla PRM but it is unclear how their approach would perform if the environment had regions that are important to compute motion plans yet difficult to sample under uniform sampling. On the other hand, our approach would identify such important regions to overcome these challenges. While these approaches~\citep{ichter2018learning,kumar2019,ichter2020}
focus on biasing the sampling distribution towards narrow areas in the environment, our approach aims to build more general high-level abstractions for the configuration space.

 ~\citet{dan_llp} use an image-based approach to learn and infer the sampling distribution using demonstrations. They use top-view images of the environment with \emph{critical regions} highlighted in the image to learn to identify critical regions. While they develop a method for predicting critical regions and using them with a low-level motion planner, they do not use these critical regions for learning abstractions and performing hierarchical planning. Our empirical results (Sec. \ref{sec:evaluation}) show that our hierarchical approach is much more effective than their non-hierarchical approach and yields significantly better performance. Additionally, their approach is also restricted to navigational problems and does not scale to configuration spaces with more than two degrees of freedom (DOFs). 

 Deep learning has also been used for learning heuristics for high-level symbolic planning. \citet{shen2020learning} use hypergraph networks for learning heuristics for symbolic planning in the form of delete-relaxation representation of the actual planning problems. \citet{karia2021learning} learn generalizable heuristics for high-level planning without explicit action representations in symbolic logic. In contrast, we focus on learning critical regions and creating high-level abstractions along with algorithms that work with these learned high-level abstractions.
 
 ~\citet{liu2020icra} use semantic information to bias the sampling distribution for navigational problems in partially known environments. Compared to it, our approach is not navigational problems and does not require semantic information explicitly but aims to learn such a notion in the form of critical regions. \emph{SPARK} and \emph{FLAME}~\citep{chamzas2020learning} use state decomposition to store past experience and use it when queried for similar state decompositions. While their approach efficiently uses the experience from previous iterations, it requires carefully crafted state decompositions in order to cover a large number of scenarios, whereas our approach generates state abstraction automatically using the predicted critical regions.

\section{Background}
\label{sec:background}
\paragraph{\textbf{Motion Planning Problem}}Let $\X = \Xf \cup \Xo $ be the configuration space of a given robot~\citep{Lav06}. Here $\Xf$ represents the set of configurations where the robot is not in collision with any obstacle and $\Xo$ represents configurations in collision with an obstacle. Let $x_i \in \Xf$ and $x_g \in \Xf$ be the initial and goal configurations of the robot. A motion planning problem is defined as follows:
\begin{definition}
A motion planning problem $\mathcal{M}$ is defined as a $4$-tuple $\langle \X, u, x_i, x_g \rangle$. where $\X = \Xf \cup \Xo$ is the configuration space and $x_i, x_g \in \Xf$ are the robot's initial and goal configurations. $u: \X \rightarrow \{0,1\}$ is a collision function that determines collisions for configurations: $u(x) = 1$ iff $x \in \Xo$.
\end{definition}

A solution to a motion planning problem is a collision-free trajectory $\tau: [0,1] \rightarrow \X$ such that $\tau(0) = x_i$ and $\tau(1) = x_g$. We abuse the notation to define membership in a trajectory as follows: For a configuration $x \in \X$, $x \in \tau$ iff there a exists a $t \in [0,1]$ such that $\tau(t) = x$. A trajectory is collision free iff $\forall x \in \tau, u(x) = 0$. 

\paragraph{\textbf{Connectivity}} A pair of low-level configurations $x_i, x_j \in \X$ is said to be connected iff there exists a collision-free motion plan between $x_i$ and $x_j$. We represent this using a connectivity function $C: \Xf \times \Xf \rightarrow \{0,1\} $: $C(x_i, x_j) = 1$ iff $x_i$ and $x_j$ are connected. Intuitively, the connectivity function $C$ represents Euclidean connectivity in $\Xf$. This is equivalent to path connectivity in configuration space for holonomic robots as each degree of freedom can be controlled independently. However, this may not be the case for non-holonomic robots as some of the motion plans may not be realizable due to their motion constraints.

We use this to define strong connectivity for the set of low-level configurations as follows:
\begin{definition}
    Let $\bar{\X} \subseteq \Xf $ be a set of configurations. $\bar{\X}$ is a strongly connected set iff for every pair of configurations $(x_i,x_j) \in \Xf \times \Xf$, $C(x_i,x_j) = 1$ and there exists a trajectory $\tau_{ij}$ such that $\tau_{ij}(0) = x_i$, $\tau_{ij}(1) = x_j$, and $\forall\,t \in [0,1], \tau_{ij}(t) \in \bar{\X}$,     
\end{definition}

\paragraph{\textbf{Critical Regions}} Our approach uses \emph{critical regions (CRs)} to generate abstractions. Intuitively, critical regions are regions in the configuration space that have a high density of valid motion plans passing through them for the given class of motion planning problems. \citet{dan_llp} define critical regions as follows:
\begin{definition}
    \label{def:critical}
    Given a robot $R$, a configuration space  $\X$, and a class of motion planning problems $M$, the measure of criticality of a Lebesgue-measurable open set $r \subseteq \X$
     is defined as $\lim_{s_n \to ^{+}r} \frac{f(r)}{v(s_n)} $, where $f(r)$ is the fraction of observed motion plans solving tasks from $M$ that pass through $s_n$, $v(s_n)$ is the measure of $s_n$ under a reference density (usually uniform), and $\to^{+}$ denotes the limit from above along any sequence $\{s_n\}$ of sets containing $r$ ($r \subseteq s_n$, $\forall\,n$). 
 \end{definition}

\paragraph{\textbf{Beam Search}}

\citet{lowerre1976harpy} introduced beam search as an optimization over breadth-first search (BFS) that explores the state space by expanding only a subset of nodes to compute a path from one node to another node in the graph. Beam search prunes the fringe to reduce the size of \emph{OPEN} set. We include pseudocode for the beam search in Appendix~A.

\section{Formal Framework}
\label{sec:formal}

We begin describing our formal framework with an example. Fig. \ref{fig:demo_reg_1} shows a set of critical regions for a given environment. Ideally, we would like to predict these critical regions and generate a state and action abstraction similar to the one shown in Fig. \ref{fig:demo_voronoi_1}. The state abstraction shown in Fig. \ref{fig:demo_voronoi_1}, similar to a Voronoi diagram, generates cells around each critical region such that the distance from each point in a cell to its corresponding critical regions is less than that from every other critical region. We call this structure a \emph{region-based Voronoi diagram (RBVD)}. Each cell in this region-based Voronoi diagram is considered an abstract state and transitions between these Voronoi cells (abstract states) define abstract actions. 

Let $\rho$ be the set of critical regions for the given configuration space $\X$. First we introduce distance metrics  $d^{c}$ and $d^r$. Here, $d^c$ defines distance between a low-level configuration $x \in \X$ and a critical region $r \in \rho$ such that $d^{c}(x,r) = \min\nolimits_{x_i \in r} d(x,x_i)$ and $d^{r}$ defines the distance between two critical regions $r_1, r_2$ such that the distance $d^{r} = \min\nolimits_{x_i \in r_1, x_j\in r_2} d(x_i, x_j)$ where $d$ is the Euclidean distance. Now, we define region-based Voronoi diagram as follows: 

\begin{definition} \label{def:RBVD}
Let $\rho = \{r_1,...,r_k\}$ be a set of critical regions for the configuration space $\X$. A region-based Voronoi diagram (RBVD) is a partition $\Psi(\rho, \X) = \{\psi_1,...,\psi_m\}$ of  $\X$ such that for every $\psi_i \in \Psi$ there exists a critical regions $r$ such that forall $x \in \psi_i$ and forall $r_j \neq r$, $d^{c}(x,r) \le d^{c}(x,r_j)$ and each $\psi_i$ is strongly connected.
\end{definition}

\paragraph{\textbf{State Abstraction}} 
We define abstract states as the Voronoi cells of an RBVD. Given an RBVD $\Psi$, labelling function $\ell: \Psi \rightarrow \Sm$ maps each cell in the RBVD to a unique abstract state $s \in \Sm$ where $|\Sm| = |\Psi|$. We use this to define the  state abstraction function $\alpha$ as follows:
\begin{definition}
    Let $R$ be the robot and $\X = \Xf \cup \Xo$ be the configuration space of the robot $R$ with set of critical regions $\rho$. Let $\Psi(\rho,\X) = \{\psi_1,..\psi_k\}$ be an RBVD for the robot $R$, configuration space $\X$, and the set of critical regions $\rho$ and let $\mathcal{S} = \{s_1,..,s_k\}$ be a set of high-level, abstract states.  We define abstraction function $\alpha: \Xf \rightarrow \mathcal{S}$ such that $\alpha(x) = s$ where $x \in \psi$ and $\ell(\psi) = s$.
\end{definition}

We extend this notation to define membership in abstract states as follows: Given a configuration space $\X = \Xf \cup \Xo$ and its set of abstract states $\mathcal{S}$ as defined above, a configuration $x \in \Xf$ is said to be a member of an abstract state $s \in \mathcal{S}$ (denoted  $x \in s$) iff $\alpha(x) = s$. We also extend the notion of strong connectivity to abstract states as follows: An abstract state $s \in \Sm$ is strongly connected iff $\ell^{-1}(s)$ is strongly connected. We now define adjacency for Voronoi cells in a region-based Voronoi diagram as follows. Recall that $C$ denotes Euclidean connectivity for configurations.

\begin{definition} \label{def:adj_RBVD}
Let $\psi_i, \psi_j$ be Voronoi cells of an RBVD $\Psi$. Voronoi cells $\psi_i$ and $\psi_j$ are adjacent iff there exist configurations $x_i, x_j$ such that $x_i \in \psi_i$, $x_j \in \psi_j$, $C(x_i,x_j) = 1$, and there exists a trajectory $\tau$ between $x_i$ and $x_j$ such that $\forall t\in[0,1], \tau(t) \in \psi_i$ or $\tau(t) \in \psi_j$.
\end{definition}

We extend the above definition to define the neighborhood for an abstract state. Two abstract states $s_i, s_j \in \Sm$ are neighbors iff $\ell^{-1}(s_i)$ and $\ell^{-1}(s_j)$ are adjacent. 

We define abstract actions as transitions between abstract states. Let $\Sm$ be the set of abstract states. We define the set of abstract actions $\mathcal{A}$ using $\Sm$ such that $\mathcal{A} = \{ a_{ij} |\forall~ (s_i, s_j) \in \Sm \times \Sm \}$. 

We now use this formulation of RBVD and state abstraction to prove the soundness of the generated abstractions . 
\begin{theorem}
\label{thm:complete}
Let $\X = \Xf \cup \Xo$ be a configuration space and $\rho$ be a set of critical regions for $\X$. Let $\Psi$ be an RBVD for the critical regions $\rho$ and the configuration space $\X$ and let $\Sm$ be the set of abstract states corresponding to $\Psi$ with a mapping function $\ell$. Let $x_0$ and $x_g$ be the initial and goal configurations of a holonomic robot $R$. If every state $s \in \Sm$ is strongly connected and there exists a sequence of abstract states $P = \langle s_{\psi_0}, ..., s_{\psi_g} \rangle$ such that $x_0 \in s_{\psi_0}$, $x_g \in s_{\psi_g}$, and all consecutive states $s_{\psi_i}, s_{\psi_{i+1}} \in P$ are neighbors, then there exists a motion plan for $R$ that reaches $x_g$ from $x_0$ with a trajectory $\tau$ such that $\tau(0) = x_0$, $\tau(1) = x_g$, and $\forall x_i \in \tau, x_i \in s_{\psi_k}$ such that $s_{\psi_k} \in P$.

\end{theorem}

\begin{proof} 
For two consecutive abstract states $s_i, s_{i+1} \in P$, let $\psi_i, \psi_{i+1}$ $\in \Psi$ be Voronoi cells such that $\ell^{-1}(s_i) = \psi_i$ and $\ell^{-1}(s_{i+1}) = \psi_{i+1}$. If $s_i$ and $s_{i+1}$ are neighbors, then according to Def. \ref{def:adj_RBVD} there exists a pair of low-level configurations $x_i, x_{i+1} \in \Xf$ such that there exists a collision free trajectory between $x_i$ to $x_{i+1}$. Def. \ref{def:RBVD} defines every Voronoi cell as a strongly connected set. Thus, for every low-level configuration $x_j \in s_i$, there exists a collision-free trajectory between $x_j$ and $x_i$ and for every low-level configuration $x_k \in s_{i+1}$, there exists a collision-free trajectory between $x_k$ and $x_{i+1}$. For a holonomic robot $R$ these trajectories should be realizable as all degrees of freedom of $R$ can be controlled independently. This implies that there exists a motion plan for $R$ between each pair of configurations in $s_{\psi_{i}}$ and $s_{\psi_{i+1}}$. 
% Therefore, considering this for each subsequent pair of abstract states in $P$, we can say that there would exist a motion plan for $R$ with a trajectory $\tau$ from $x_0$ to $x_g$ such that $\tau(0) = x_0$, $\tau(1) = x_g$, and $\forall x_i \in \tau, x_i \in s_{\psi_k}$ such that $s_{\psi_k} \in P$.

\end{proof}

Theorem \ref{thm:complete} proves that the computed abstractions would be sound as well as satisfy downward refinement property for holonomic robots. The proof does not hold for non-holonomic robots as the low-level trajectories may not be realizable given their motion constraints. However, the algorithm developed below is probabilistically complete for all robots and performed well for non-holonomic robots in our empirical evaluation.

\section{Learning Abstractions and Planning}
\label{sec:algo}

Our approach computes hierarchical state and action abstractions using predicted critical regions and uses them efficiently for hierarchical planning. Now, we first discuss our approach for generating and using hierarchical abstractions using critical regions (Sec.~\ref{sub:abstractions}) and then we discuss how we learn these critical regions (Sec.~\ref{sub:learn}).

\begin{algorithm}[t!]
    \caption{Hierarchical Abstraction-guided Robot Planner (HARP)}
    \label{alg:harp}
    \KwIn{Configuration space $\X$, a region predictor $\Phi$, an initial configuration $x_0 \in \X$, goal configuration $x_g \in \X$, a custom heuristic $h$, low-level sampling-based motion planner \emph{MP}}
    \KwOut{A motion plan $\tau$}
    
    $\rho$ $\gets$ predict\_critical\_regions($\Phi$, $\X$, $x_0$, $x_g$) \\
    $\Sm, \mathcal{A} \gets$ generate\_state\_action\_abstractions($\rho$, $\X$ )\\ 
    $s_0, s_g \gets$ get\_HL\_state($\Sm$, $\rho$, $x_0$), get\_HL\_state($\Sm$, $\rho$, $x_g$) \\
    $\mathcal{P}$ $\gets$ multi-source\_bi-directional\_beam\_search($\Sm$, $\mathcal{A}$, $s_0$, $s_g$) \\
    $\tau \gets $ refine\_path($\mathcal{P}$, \emph{MP}) \\ 
    $h \gets$ update\_heuristic($\tau$, $\Sm$) \\ 
    return $\tau$
    \end{algorithm}

    \Description{ 
        Algorithm 1 for HARP. This algorithm summerizes our approach for constructing abstractions from automatically identified critical regions and using them to perform hierarchical planning using the multi-source bi-directional beam search and a probabilistically complete low-level motion planner. 
    }
\subsection{Generating and Using  Abstractions}
\label{sub:abstractions}

In this section, we describe our approach -{}- \emph{\textbf{H}ierarchical \textbf{A}bstraction-guided \textbf{R}obot \textbf{P}lanner (HARP)} -{}- for generating abstract states and actions and using them to efficiently perform hierarchical planning. A na\'ive approach would be to generate a complete RBVD and then extract abstract states and actions from it. This would require iterating over all configurations in the configuration space and computing a large number of motion plans to identify executable abstract actions. This is expensive (and practically infeasible) for continuous low-level configuration spaces. Instead, we use the RBVD as an implicit concept. We generate abstractions on-the-fly by computing membership of low-level configurations in abstract states only when needed. 

Vanilla high-level planning using the set of all abstract actions $\mathcal{A}$ would be inefficient as it may yield plans for which low-level refinement may not exist as we do not know the applicability of these abstract actions at low-level. To overcome this challenge, we develop a hierarchical multi-source bi-directional planning algorithm that performs high-level planning from multiple abstract states. Generally, a multi-source approach would not work for robot planning because it is not clear what the intermediate states are. In this paper, we use critical regions as abstract intermediate states and utilize a multi-source search. This utilizes learned information better than single source and single direction beam search. Our high-level planner generates a set of candidate high-level plans from the abstract initial state to the abstract goal state using a custom heuristic (which is continually updated). These paths are then simultaneously refined by a low-level planner to compute a trajectory from the initial low-level configuration to the goal configuration while updating the heuristic function.

Algorithm \ref{alg:harp} describes our approach for generating and using hierarchical abstractions. Given the configuration space $\X$ and initial and goal configurations ($x_0$ and $x_g$) of the robot $R$, HARP uses a learned DNN $\Phi$ to generate a set of critical regions $\rho$ (Sec.~\ref{sub:learn} discusses how we learn this model $\Phi$) (line $1$). The remainder of Alg.~\ref{alg:harp} can be broken down into three important steps: $1)$ computing a set of candidate high-level plans, $2)$ refining candidate high-level plans into a low-level trajectory, and $3)$ updating the heuristic for abstract states. We now explain each of these steps in detail.

\begin{algorithm}[t!]
    \caption{Multi-source Bi-directional beam search}
    \label{alg:msbi_beam_search}
    \KwIn{Set of states $\Sm$, set of actions $\mathcal{A}$, initial state $s_0$, goal state $s_g$, distance function $h'$, beam width $w$, number of high-level plans $N$}
    \KwOut{A set of $N$ paths from $s_0$ to $s_g$}
    
    fringe = PriorityQueue() \\ 
    fringe.add($s_0$), fringe.add($s_g$) \\
    $\bar{S} \gets$ sample\_states($\Sm$) \\ 
    solutions = Set() \\ 
    \ForEach{$s \in \bar{S}$}{
        fringe.add($0$,$(s,\text{None},\text{Set()})$) \\
    }
    \While{$N > 0$ and fringe is not empty}{
        working\_fringe $\gets$ select top $w$ nodes from the  fringe\\
        empty\_fringe(fringe) \\ 
        \While{working\_fringe is not empty}{
            current, path, visited $\gets$ working\_fringe.pop() \\
            \If{current = $s_g$}{
                add path to solutions \\ 
                $N \gets N - 1$ \\ 
            }
            \Else{
                path.add(current) \\ 
                visited.add(current) \\ 
                \ForEach{ node $\in $ current.successors }{
                    \If{node $\notin$ visited }{
                        $p \gets h'(\text{current},\text{node}) + min\{h'(\text{node},s_0), h'(\text{node},s_g) \}  $ \\ 
                        fringe.add($p$, (current, path, visited))                 }
                    % \Else{
                    %     merge paths for current and node 
                    % }
                }
                    
            }
        }
    }
    return solutions
    \end{algorithm}

\subsubsection{\textbf{Computing High-Level Plans}} To compute high-level plans that reach the goal configuration $x_g$ from the initial configuration $x_0$, first we determine abstract initial and goal states $s_0$ and $s_g$ corresponding to the initial and goal configurations $x_0$ and $x_g$ (line $3$). To do this efficiently, we store sampled points from each critical region in a k-d-tree and query it to determine the abstract state for the given low-level configurations without explicitly constructing the complete RBVD. This allows us to dynamically and efficiently determine high-level states for low-level configurations. 

Once we determine initial and goal states $s_0$ and $s_g$, we use our high-level planner to compute a set of candidate high-level plans going from $s_0$ to $s_g$ (line $4$). To compute these candidate high-level plans, we develop a multi-source bi-directional variant of beam search~\citep{lowerre1976harpy} that yields multiple high-level candidate plans. We call this \emph{multi-source bi-directional beam search}. Alg.~\ref{alg:msbi_beam_search} presents the psudocode for mulit-source bi-directional beam search.

Intuitively, this algorithm works as follows: we use a priority queue to maintain a fringe to keep track of the current state of the search. We initialize this fringe with multiple randomly sampled abstract states (line $5$) to allow the beam search to start from multiple sources. We select and expand these nodes in a specific order (line $11$) to compute high-level plans that reach from the initial state to the goal state. 

Formally, the nodes in the fringe are expanded as follows: Let $n$ be a node in the fringe. We select the node to expand using $f(n) = g(n) + h(n)$, where $g(n)$ is  the cost of the path heading to $n$ (unit cost for each action) and $h(n)$ is a custom heuristic that is defined as follows. Let $m$ be the parent of  $n$ and let $s_n$ and $s_m$ be the abstract states corresponding to the nodes $n$ and $m$. The heuristic $h(n)$ is computed as: 
$$h(n) = h'(s_m,s_n) + min\{h'(s_n,s_i), h'(s_n,s_g) \}  $$ 
Here, $h'(s_1,s_2) = \epsilon_{12} d^{r}(r_1,r_2)$ defines the estimated distance between abstract high-level states $s_1$ and $s_2$ with corresponding critical regions $r_1$ and $r_2$ respectively. $s_i$ is the initial abstract state and $s_g$ is the goal abstract state.

$\epsilon_{ij} \in (0,1]$ is a constant that accounts for imprecise abstract actions. Alg.~\ref{alg:harp} dynamically changes it to update the heuristic function (HARP line $6$) and making it more accurate. Initially, $\epsilon_{ij}$ is set to $1$ for all $i$ and $j$. Once a low-level trajectory $\tau$ is computed (explained later), we compute the abstraction $\bar{\tau}$ of this trajectory. For each consecutive pair of abstract states $\langle s_i, s_{j} \rangle$ in $\bar{\tau}$, we decrease the value of $\epsilon_{ij}$ by $\epsilon_{ij}/2$ . This allows HARP to use experience from previously computed trajectories to prioritize abstract actions that have low-level refinements to compute accurate high-level plans.

After an abstract state $s_j$ is selected from the fringe (Alg.~\ref{alg:msbi_beam_search}, line $8$), its successors are created by applying abstract actions on it and adding these successors to the fringe (Alg.~\ref{alg:msbi_beam_search} lines $18$-$21$). This process continues until $N$ high-level plans are found.

Once we generate a set of candidate high-level plans from multi-source bi-directional beam search, we use a low-level planner to refine these plans into a low-level collision-free trajectory from initial low-level configuration $x_0$ to goal configuration $x_g$ (line $5$). 

\subsubsection{\textbf{Refining High-level Plans}}

While any \emph{probabilistically complete} motion planner can be used to refine the computed high-level plans into a low-level trajectory between given two configurations, we use \emph{Learn and Link Planner (LLP)}~\citep{dan_llp} as a low-level planner in HARP (\emph{MP} in Alg.~\ref{alg:harp}) as it allows us to easily use the abstract states and their critical regions to efficiently compute motion plans. LLP is a sampling-based motion planner that initializes exploration trees rooted at $N$ samples from the configuration space and extends these exploration trees until they connect and form a single tree. Once a single tree is formed, the planner uses Dijkstra's Algorithm~\citep{dijkstra1959note} to compute a path from the initial state to the goal state.

 In order to use LLP to refine a set of candidate high-level plans simultaneously, we first select a subset of critical regions $\bar{\rho} \subseteq \rho$ that includes all critical regions corresponding to all high-level states in candidate plans. We use this subset of critical regions $\bar{\rho}$ to provide initial samples to initialize exploration trees of LLP. In this work, we generate $M$ samples from the set of critical regions $\bar{\rho}$ and generate the rest of the $N-M$ samples using uniform random sampling. Similarly, to expand these exploration trees, we generate a fixed number of samples from the set critical regions $\rho$ and then continue with uniform sampling.

We use these characteristics of our algorithm to show that our approach is \emph{probabilistic complete}.

\begin{theorem}
    \label{thm:prob}
    % HARP is probabilistic complete.
    If the low-level motion planner (\emph{MP} in Alg. \ref{alg:harp}) is probabilistically complete, then HARP is probabilistically complete.
\end{theorem}

\begin{proof}
    (Sketch) While refining high-level plans to a low-level motion plan (line $5$ in Alg. \ref{alg:harp}), HARP uses a fixed number of samples from the critical regions along the high-level plans to initialize a low-level motion planner. This does not reduce the set of support (regions with a non-zero probability of being sampled) of the sampling distribution being used by the motion planner.
\end{proof}

\renewcommand{\thesubfigure}{(\Alph{subfigure})} 
\begin{figure}[t]
    \Alph{subfigure}
            \centering
              \subfigure[]{\includegraphics[height=0.75in]{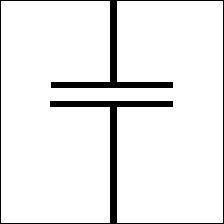}}      
                \subfigure[]{\includegraphics[height=0.75in]{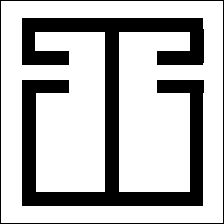}} 
                \subfigure[]{\includegraphics[height=0.75in]{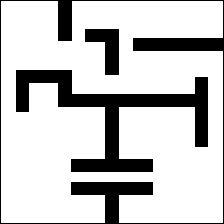}}      
                \subfigure[]{\includegraphics[height=0.75in]{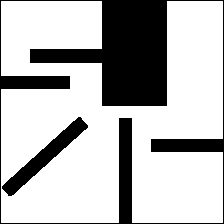}} \\ 
              \subfigure[]{\includegraphics[height=1in]{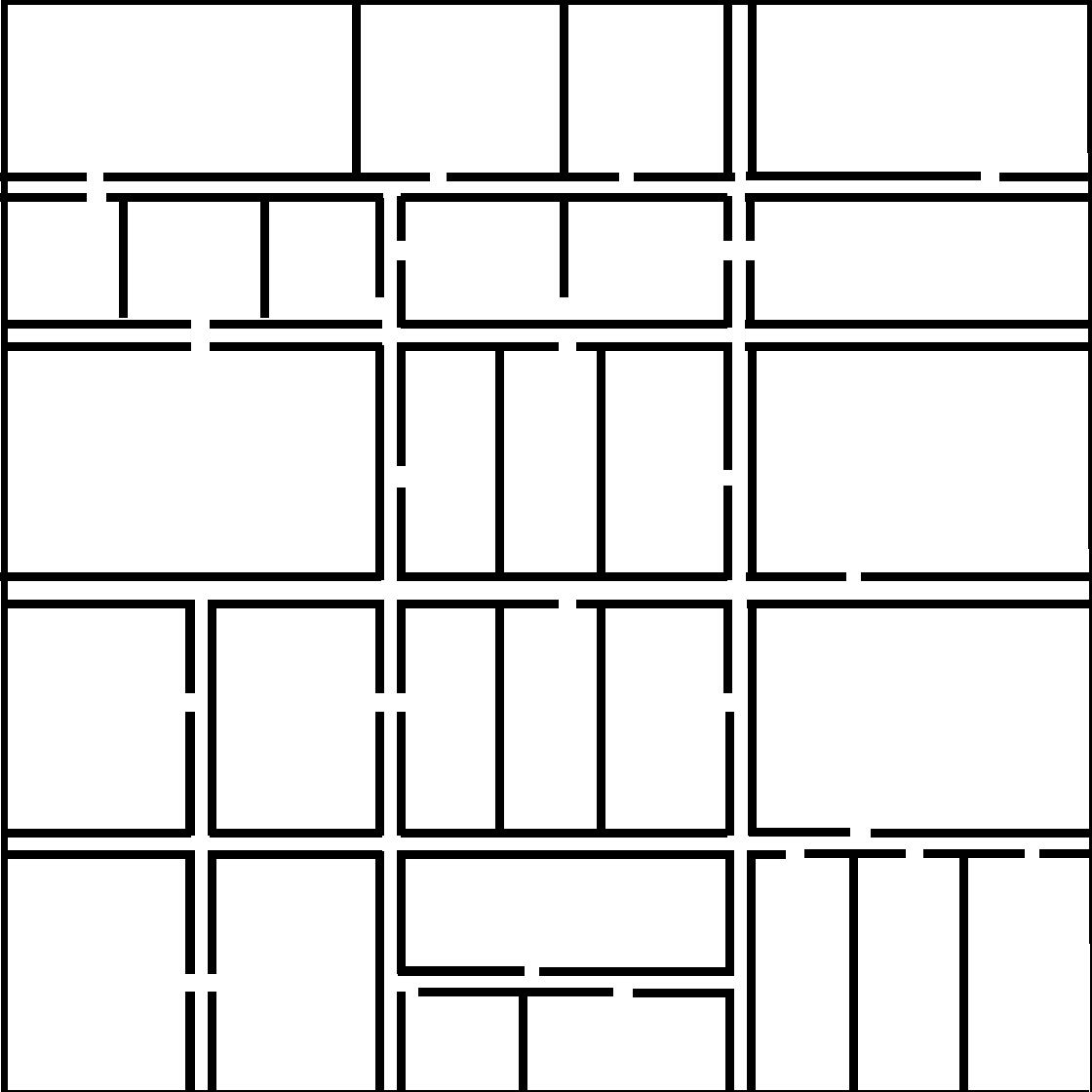}}
              \subfigure[]{\includegraphics[height=1in]{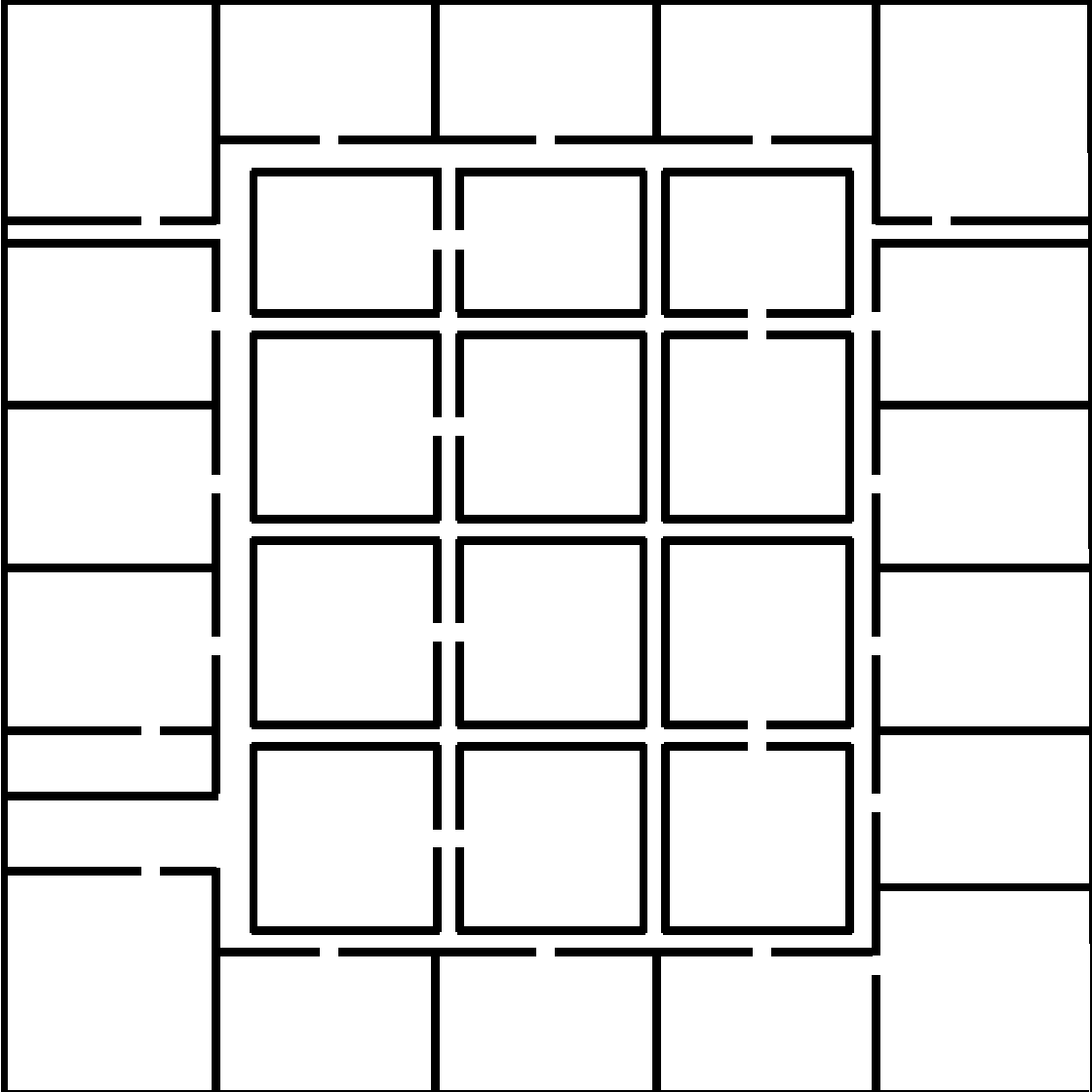}} 
              \subfigure[]{\includegraphics[height=1in]{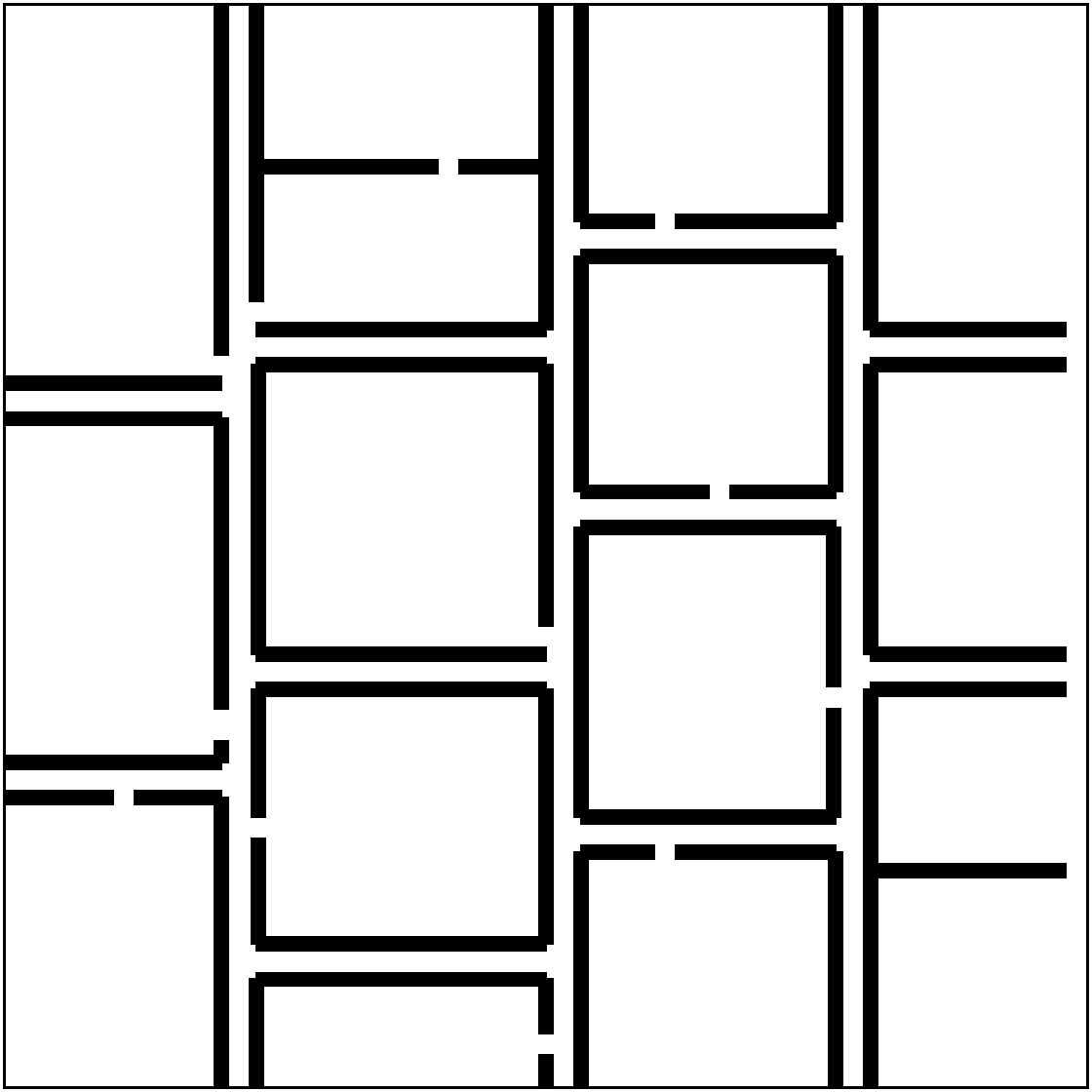}} \\ 
              \setcounter{subfigure}{8}
              \subfigure[]{\includegraphics[height=1in]{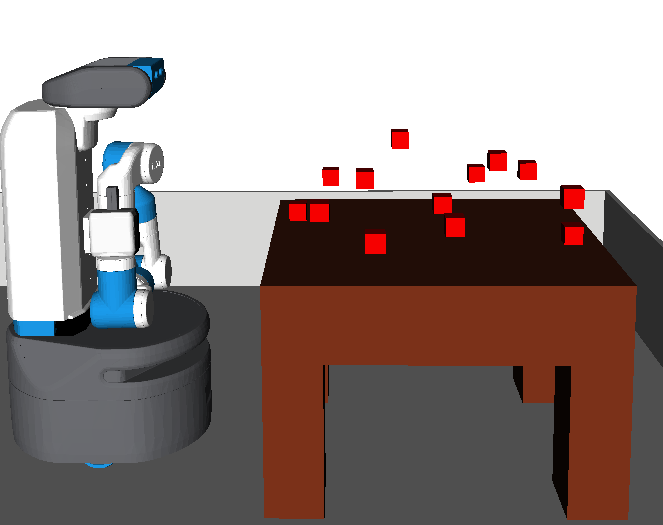}}
              \subfigure[]{\includegraphics[height=1in]{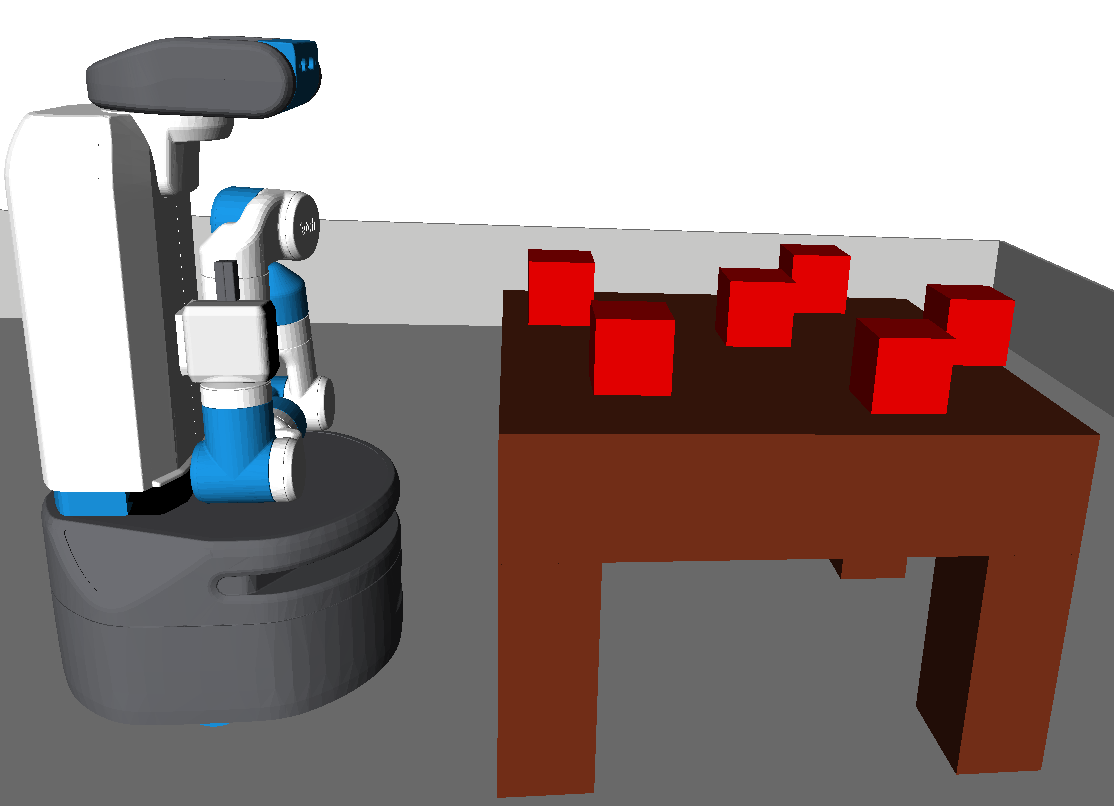}} 
          \caption{Test environments for our approach. Dimensions of environments (A)-(D) are $\mathbf{5m \times 5m}$ and dimensions of environments (E)-(G) are $\mathbf{25m \times 25m}$. (A)-(G) are used for navigational problems while (H) and (I) are used for manipulation problems.}
          \label{fig:envs}

          \Description{Test environments for our approach: The figure shows test environments for our approach. A to G are used to test navigational planning problems and I and G are used to test manipulation planning problem.  A to D are 5 meters by 5 meters while E to G are 25 meters by 25 meters.}
      \end{figure}
\renewcommand{\thesubfigure}{(\alph{subfigure})} 

\subsection{Learning to Predict Critical Regions}
\label{sub:learn}

We now present our approach for learning a model $\Phi$ that predicts critical regions for a given environment. We first explain how our approach is able to generate a robot-specific architecture of the network and then describe the training process. 

 \subsubsection{\textbf{Deriving Robot Specific Network Architectures}}
 \label{sec:network}
We use the standard fully convolutional UNet architecture~\citep{Ronneberger2015unet} as our base network. Appendix~B includes network architecture for this network. We use the robot's geometry and its number of DOFs to derive a robot specific architecture as follows:  

Let $n$ be the number of DOFs of the robot and let $k$ be the number of DOFs that are not determined by the location of the robot's end-effector in the workspace. For manipulation problems, we consider gripper of the robot as the its end-effector and for navigational problems, we consider the robot's base link as its end effector. First, we use these parameters to update the base UNet architecture. We use the last layer of the base architecture to predict critical regions for the end-effector's location and include $k$ additional convolutional layers to predict critical regions for each $k$ DOFs that are not determined by the location of the robot's end effector.
\begin{figure}[t]

        \centering
        \subfigure[]{\includegraphics[width=1in]{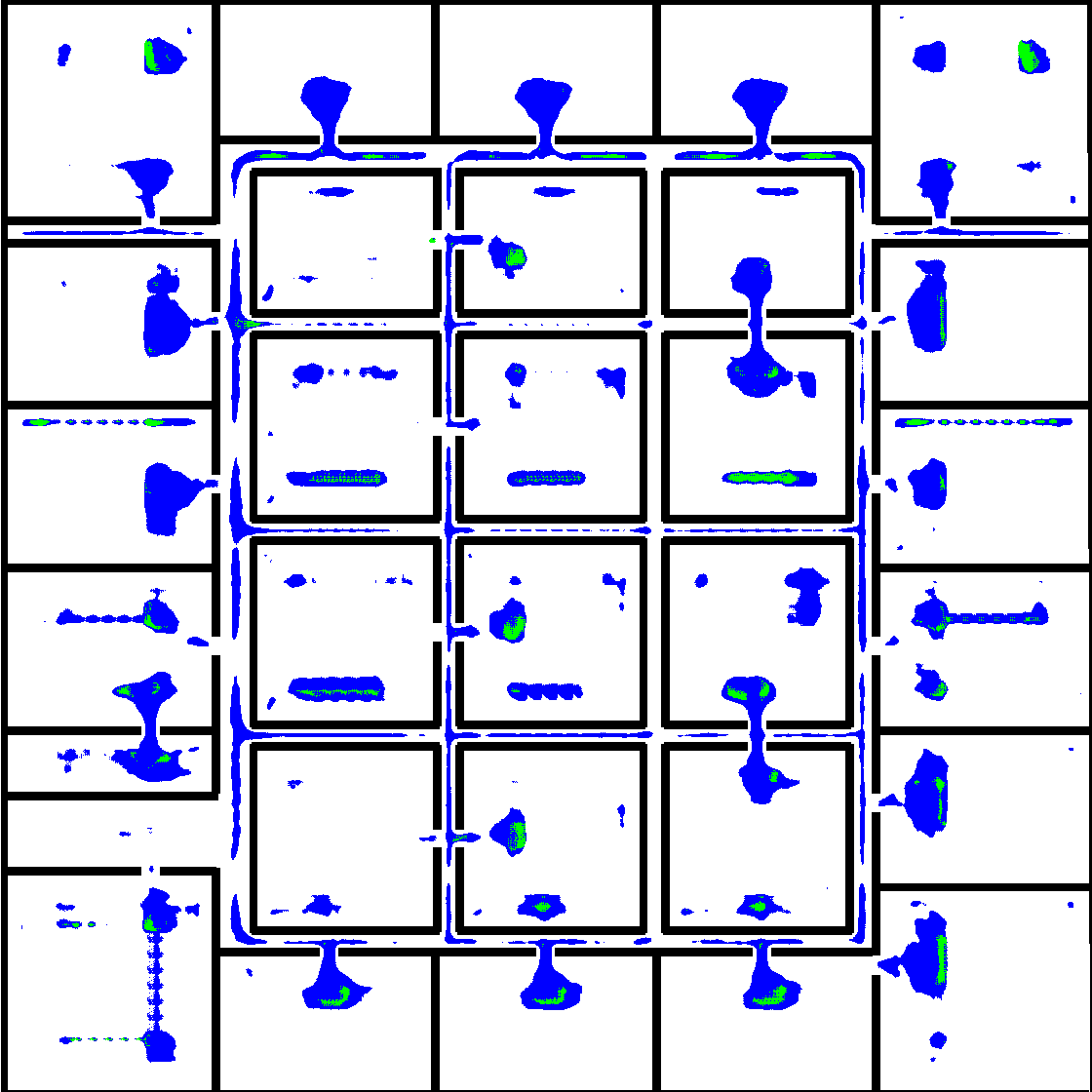}}  
        \subfigure[]{\includegraphics[width=1in]{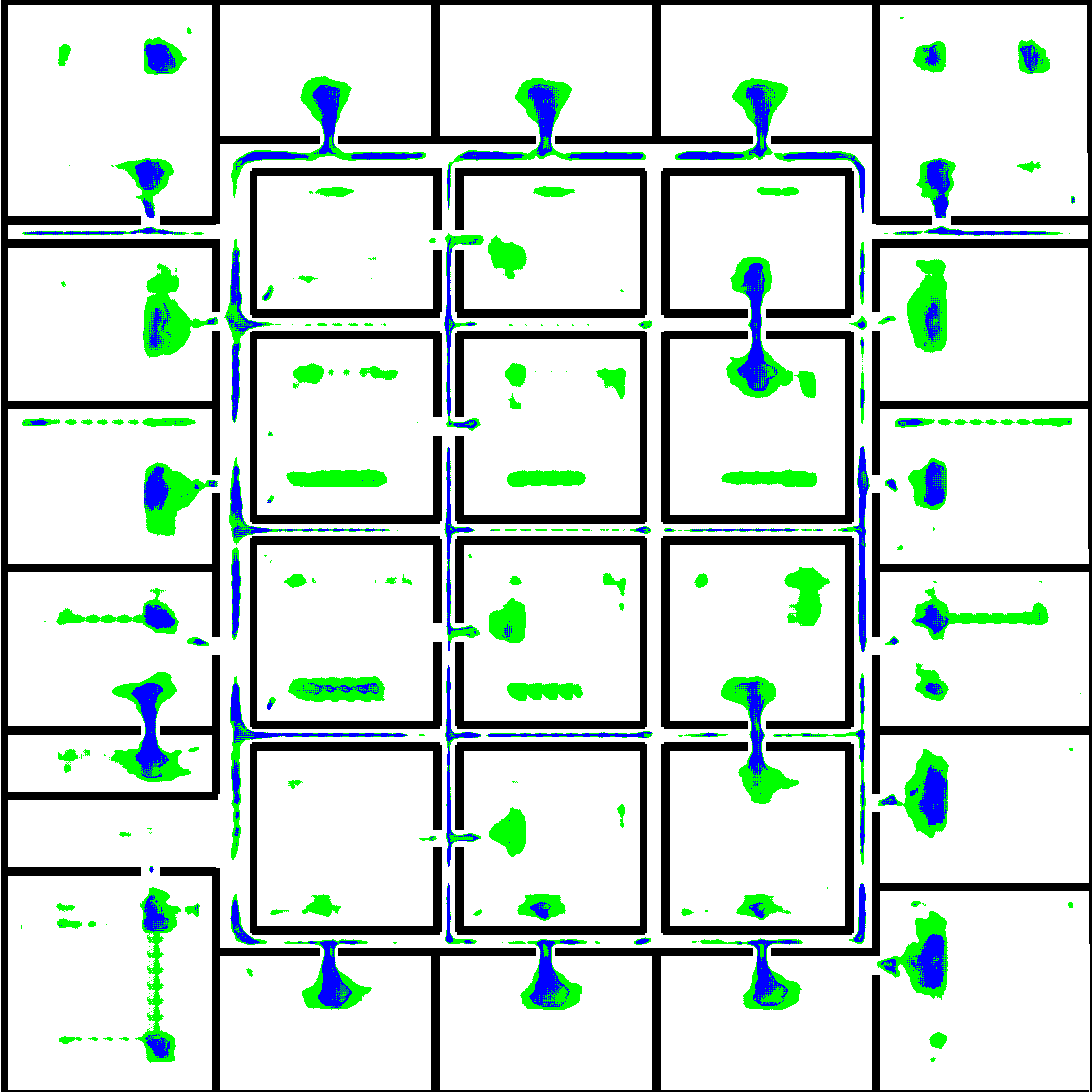}}
        \subfigure[]{\includegraphics[width=1in]{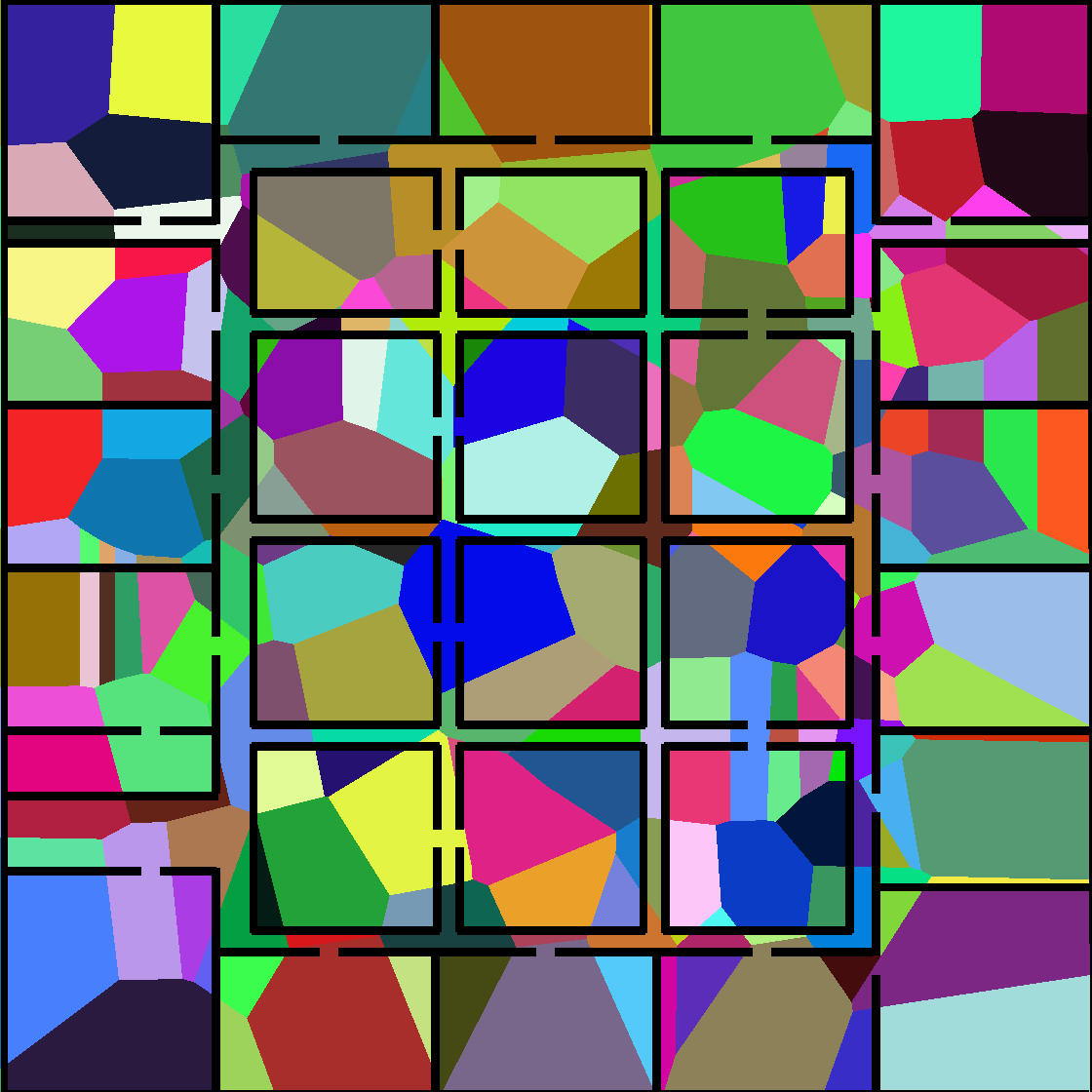}}
       \caption{Critical regions and generated abstraction for a $\mathbf{4}$-DOF hinged robot. (a) and (b) Critical regions predicted by the model. Blue regions in (a) show that model predicted the robot's base link to be horizontal while green regions show that the model predicted the robot's base link to be vertical. Blue regions in (b) show that the network predicted the hinge to be closer to $\mathbf{180^{\circ}}$ and green regions show that the network predicted it to be  closer to $\mathbf{90^{\circ}}$ or $\mathbf{270^{\circ}}$. (c) 2D projections of state abstraction generated by our approach though our approach does not need to explicitly generate these abstractions.}
       \label{fig:se3_results}
       \Description{Critical regions identified by our approach for a hinged robot wiht 4 degrees of freedom and abstractions constructed by our approach.}
    %    \vspace{-1em}
   \end{figure}

The input to the network is a tensor of dimension four. The size of the first three dimensions of the input tensor depends on the number of bins used to discretize the environment (which can be arbitrary). The number of channels in the input tensor is determined using the parameter $n$. If the robot has $n$ DOFs, then the input tensor would have a total of $n+1$ channels. The first channel in the input represents the occupancy matrix of the environment. It is generated by performing a raster scan of the environment. The rest of the $n$ channels represent goal values for each DOF of the robot -{}- one for each DOF of the robot. 

Similarly, each label is a tensor of dimension four. The size of the first three dimensions is similar to the input tensor. The number of channels in the label tensor is also computed using the robot's geometry. For a robot with $k$ DOFs that are not determined by the location of robot's end effector in the workspace, the label tensor would have a total of $k+1$ channels. The first channel represents critical regions for the end-effector's location in the workspace and the rest of the $k$ channels represent critical regions for the $k$ DOFs that are not determined by this location -{}- one channel for each of the $k$ DOFs of the robot.

   E.g, consider a $5$-DOF hinged robot.  The robot's $5$ DOFs are $(x,y,z,\theta,\omega)$ where $x$, $y$, and $z$ represent the location of the robot's base link in the workspace, $\theta$ represents the rotation of the base link, and $\omega$ represents the hinged angle. So for this robot, $n$ would equal to $5$ and $k$ would equal to $2$ as only the base rotation $\theta$ and the hinged angle $\omega$ are not determined by the location of the robot's end-effector (base link in this case) in the workspace.  So according to the previous discussion, the network would contain $k=2$ additional layers to predict critical regions for $\theta$ and $\omega$. The input tensor would have a total of $n+1=6$ channels and the label tensor would have a total of $k+1 =3$ channels. 
   \begin{figure}[t!]
    \vspace{-1em}
        \centering
        \subfigure[]{\includegraphics[height=1.15in]{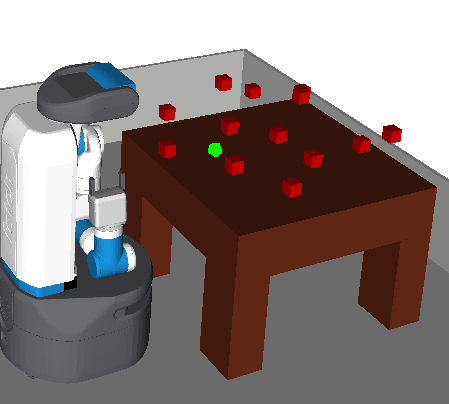}}
           \subfigure[]{\includegraphics[height=1.15in]{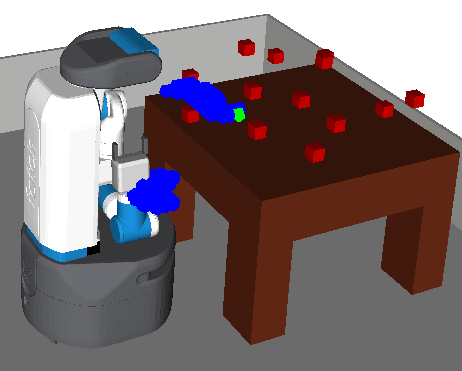}}
       \caption{Critical regions for arm manipulation task using an $\mathbf{8}$-DOF Fetch. The green region in (a) shows the goal location for the end effector. (b) shows the critical regions generated by the learned model. Although the network predicts CRs for all the joints, only CRs for end-effector's location are shown.}
       \label{fig:fetch_results}
       \Description{Critical regions identified by our approach for a fetch robot with 8 degrees of freedom.}
    %    \vspace{-1.5em}
   \end{figure}

  \begin{figure*}[t]
    ~ \\
        \begin{center}
             \includegraphics[width=\textwidth]{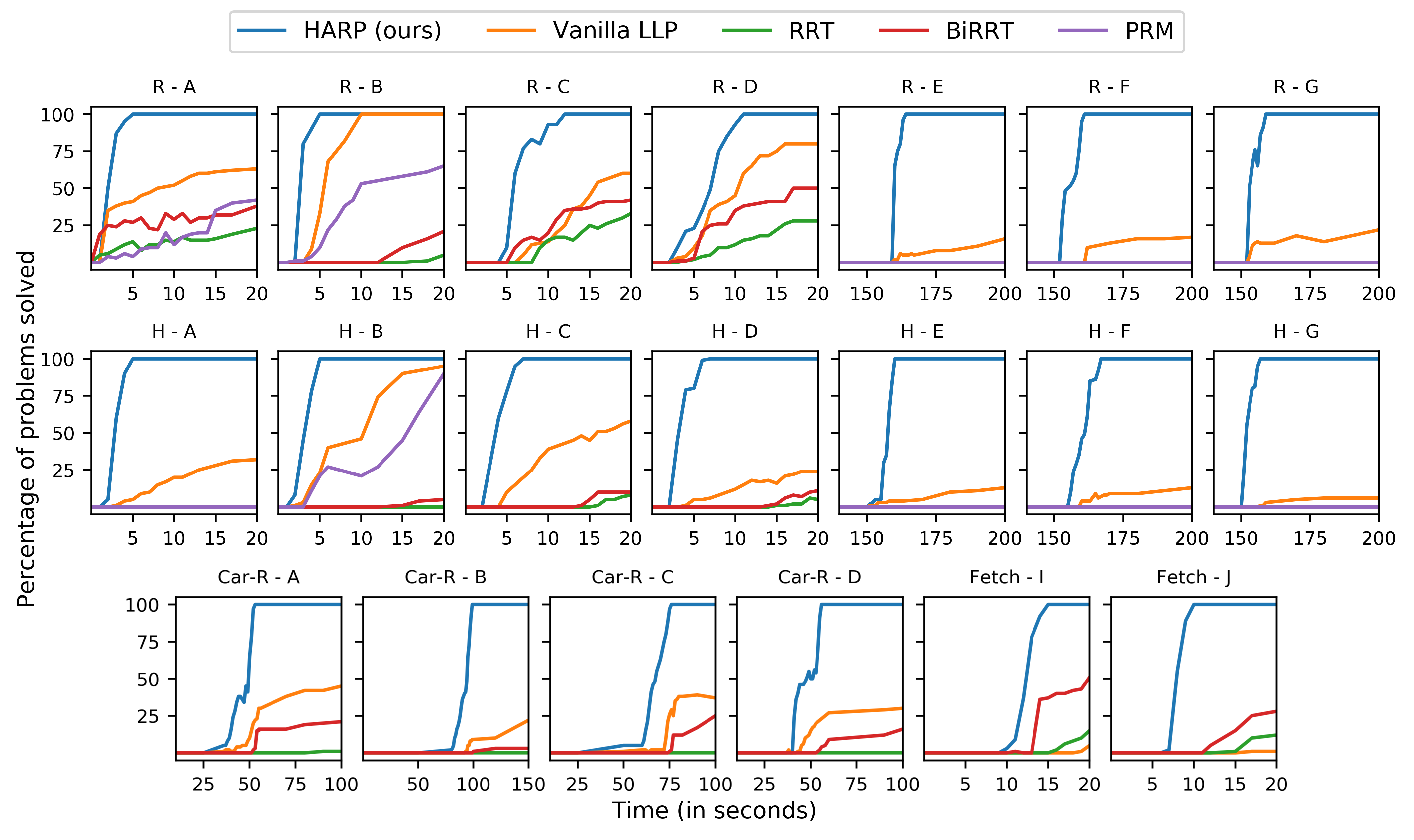}
        \end{center}
        \caption{Each plot shows the fraction of $\mathbf{100}$ independently generated motion planning tasks solved (y-axis) in the given time (x-axis) for all the test environments and robots. The title of each subplot represents the robot and the environment. E.g., ``R - A'' stands for rectangular robot in environment A (Fig. \ref{fig:envs}(A)) and ``H - A'' stands for hinged robot in environment A. }
        \label{fig:results}
        \Description{Comparision with baselines: The figure has 20 plots. The figure shows qualitative comparision of our appraoch with sampling-based and learning-based state-of-the-art motion planners.}
    \end{figure*}

\subsubsection{\textbf{Training the Network}} 
The layer predicting critical regions for the end-effector's locations in the workspace, $(L_l)$,  uses the sigmoid activation as the task is similar to element-wise classification. The layers predicting critical regions for the rest of the degrees of freedom that are not determined by the end-effector's location, $(L_{i})$ use the softmax activation as the task corresponds to multi-class classification. The loss function is defined as follows:

\begin{dmath*}
    \mathcal{L} = \mathcal{L}_{L_1} + \sum_{i=0}^{k} \mathcal{L}_{L_i}
\end{dmath*}
where, $\mathcal{L}_{L_1}$ is weighted log loss and $\mathcal{L}_{L_i}$ is softmax cross entropy loss for $i^{\text{th}}$ degree of freedom not determined by  end-effecot's location in the workspace. We use \emph{ADAM Optimizer} \citep{kingma2014adam} with learning rate $10^{-4}$. We implement the UNet architecture shown in the appendix using Tensorflow \cite{abadi2016tensorflow}. We train the network for $50,000$ epochs.

\section{Empirical Evaluation}
\label{sec:evaluation}

We extensively evaluate our approach in twenty different scenarios with four different robots. All experiments were conducted on a system running \emph{Ubuntu 18.04} with \emph{8-core i9} processor, \emph{32 GB RAM}, and an \emph{Nvidia 2060} GPU (our approach uses only a single core)
 OpenRAVE robot simulator~\citep{diankov10_openrave}.
  We compare our approach with state-of-the-art motion planners such as \emph{RRT}~\citep{lavalle1998rapidly}, \emph{PRM}~\citep{kavraki1996probabilistic}, and \emph{BiRRT}~\citep{kuffner2000rrt}. As LLP is implemented using \emph{Python}, we use the \emph{Python} implementation of the baseline algorithms available at \url{https://ompl.kavrakilab.org/} for comparison. Our training data, python implementation, trained models, and results are available at~\url{https://aair-lab.github.io/harp.html}.

\subsubsection{\textbf{$3$-DOF Rectangular Robot (R)}}
For the first set of experiments, the objective is to solve motion planning problems for a $3$-DOF rectangular robot. The robot can move along the $x$ and $y$ axes and it can rotate around the $z$ axis. 

\subsubsection{\textbf{$3$-DOF Non-Holonomic Rectangular Car Robot (Car)}} For the second set of experiments, we evaluated our approach with a rectangular non-holonomic robot similar to a simple car. Controls available to operate the robot were linear velocity $v \in [-0.2, 0.2]$ and the steering angle $\theta \in [-\frac{\pi}{4},\frac{\pi}{4}] $ while three degrees of freedom (location along $x$-axis, location along $y$-xis, and rotation around $z$-axis) were required to represent the robot's transformation. 

\subsubsection{\textbf{$4$-DOF Hinged Robot (H)}} For the third set of experiments, we used a robot with a hinge joint to evaluate our approach. The robot's $4$ DOFs are its location along $x$ and $y$ axes, rotation along $z$-axis ($\theta$), and the hinge joint ($\omega$) with the range $[-\frac{\pi}{2}, \frac{\pi}{2}]$.

\subsubsection{\textbf{8-DOF Fetch Robot}} For the last set of experiments, we used our approach with a mobile manipulator named Fetch~\cite{wise16_fetch} to perform arm manipulation. The goal of this experiment is to evaluate the scalability of our approach to robots with high degrees of freedom. 

\subsubsection{\textbf{Generating Training Data}} We use $20$ training environments to generate training data for navigational problems (robots $R$, Car, and $H$) and $6$ training environments for manipulation problems (Fetch robot). To generate training data for each training environment $E$, we randomly sample a set of $100$ configurations $\mathcal{G}$, which would serve as goal states for the motion planning problems. For each goal configuration $g_i \in \mathcal{G}$, we sample a set of $50$ initial states $\mathcal{I}$. We use an off-the-shelf motion planner to compute motion plans for these inital and goal states and combined solutions to generate critical regions for each goal state. We use these motion plans to compute critical regions for the given pair of environment $E$ and the goal configuration $g$ using the Def.~\ref{def:critical}. We use OpeanRAVE robot simulator~\citep{diankov10_openrave} and OMPL's implementation of BiRRT~\citep{kuffner2000rrt} to generate the training data. 

To generate the input vector, we discretize the environment into $n_d$ bins. This also implies that the degrees of freedom of the robot that are determined by the robot's end-effecotr's location in the workspace are also discretized into $n_d$ bins. We discretize the rest of the degrees of freedom that are not determined by the robot's end-effecotr's location into $p$ bins to generate input and label tensors according to the discussion in Sec.~\ref{sec:network}. We augment the computed tensors by rotating them by $90^\circ$, $180^\circ,$ and $270^\circ$ to obtain more training samples. We also omit an additional dimension from the tensors for navigational problems as we fix the robot's $z$-axis for these problems. The table below shows the training details for each robot.

\begin{table}[h!]
  \footnotesize
  \centering
  \begin{tabular}{|l|c|c|l|l|c|c|}
  \hline
  Robot  & $n_d$ & $p$  & \begin{tabular}[c]{@{}l@{}}Input \\ Shape\end{tabular} & \begin{tabular}[c]{@{}l@{}}Label \\ Shape\end{tabular} & |E| & \# Samples \\ \hline
  $R$, Car  & $224$ & $4$       &  (224,224,4)                                             & (224,224,2)                                           & 20                                                                                       & 8000                                                                               \\ \hline
  $H$ & $224$ & $5$      & (224,224,3)                                             & (224,224,21)                                           & 20                                                                                       & 8000                                                                               \\ \hline
  Fetch & $64$ & $10$ & (64,64,64,11)                                           & (64,64,64,9)                                          & 6                                                                                        & 720                                                                                \\ \hline
  \end{tabular}
  \end{table}

\subsubsection{\textbf{Evaluating the Approach}} Figures \ref{fig:envs} and \ref{fig:fetch_results} show the test environments (unseen by the model while training) for our system. Environments shown in Fig. \ref{fig:envs} are inspired by the indoor office and household environments. Our training data consisted of $20$ environments similar to the ones shown in Fig. \ref{fig:envs}(A)-(D) with dimensions $5m \times 5m$. We investigate the scalability of our approach by conducting experiments in environments shown in Fig. \ref{fig:envs}(E)-(G) with dimensions $25m \times 25m$ (much larger than training environments). To handle such large environments without making any changes to the DNN, we use the standard approach of sliding windows with stride equal to window-width~\citep{window_1,window_2,window_3,window_4}. This crops the larger environment into pieces of the size of the training environments. Individual predictions are then combined to generate a set of critical regions for arbitrary large environments. We also evaluate the applicability of our approach to non-holonomic robots in environments shown in Fig. \ref{fig:envs}(A)-(D).

% \vspace{-1.em}
\subsection{Analysis of the Results}
As discussed in the introduction (Sec.~\ref{sec:intro}), our objective is to show whether $1)$ state and action abstractions can be derived automatically and $2)$ whether auto-generated state and action abstractions can be efficiently used in a hierarchical planning algorithm. Additionally, we also investigate $3)$ does dynamically updating the heuristic function (line $6$ in Alg. \ref{alg:harp}) improve Alg.~\ref{alg:harp}'s efficiency?

\subsubsection{\textbf{$1)$ Can We Learn State and Action Abstractions?}}
Our approach learns critical regions for each DOF of the robot. Fig. \ref{fig:se3_results} shows critical regions predicted by our learned model for the hinged robot $H$. We can see that our model was able to identify critical regions in the environment such as doorways and narrow hallways. Fig. \ref{fig:se3_results}(a) shows critical regions for orientation of the robot's base link (captured by DOF $\theta$). The blue regions in the figure represent the horizontal orientation of the robot and the green regions represent the vertical orientation of the robot. Fig. \ref{fig:se3_results}(b) shows critical regions for the hinge joint $\omega$ for the robot $H$. Here, blue regions show that the network predicted the hinge joint to be flat (close to $0^{\circ}$) and green regions represent configurations where the model predicted "L" configurations of the robot ($\omega$ close to $90^\circ$ or $270^{\circ})$. Fig. \ref{fig:se3_results} shows that our approach was able to predict the correct orientation of the robot accurately most of the time. Our approach was able to scale to robots with a high number of degrees of freedom. Fig. \ref{fig:fetch_results}(a) shows one of the test environments used for these experiments and Fig. \ref{fig:fetch_results}(b) shows the predicted critical regions. This shows that our model was able to learn critical regions in the environment that can be used to generate efficient abstractions. We include similar results for other environments in Appendix~C.

Now we answer the second question on whether using abstractions to compute motion plans helps improve the planner's efficiency by qualitatively comparing our approach with a few existing sampling-based motion planners.

\subsubsection{\textbf{$b)$ Can Learned Abstractions be Used Efficiently for Hierarchical Planning?}}

We compare our approach against widely used SBMPs such as RRT~\citep{lavalle1998rapidly}, PRM~\citep{kavraki1996probabilistic}, and  BiRRT~\citep{kuffner2000rrt}. 

Fig. \ref{fig:results} shows the comparison of our approach with other sampling-based motion planners. The x-axis shows the time limit in seconds and the y-axis shows the percentage of motion planning problems solved in that time limit. For each time limit on x-axis, we randomly generate $100$ new motion planning problems to thoroughly test our approach and reduce statistical inconsistencies. Fig. \ref{fig:results} shows that our approach significantly outperforms all of the existing sampling-based motion planners. Specifically for environments $C$, $D$, and $E$, uniform sampling-based approaches were not able to solve a single problem in a time threshold of $600s$. 

\begin{figure}
    \centering

    \subfigure[]{
    \includegraphics[width=0.48\columnwidth]{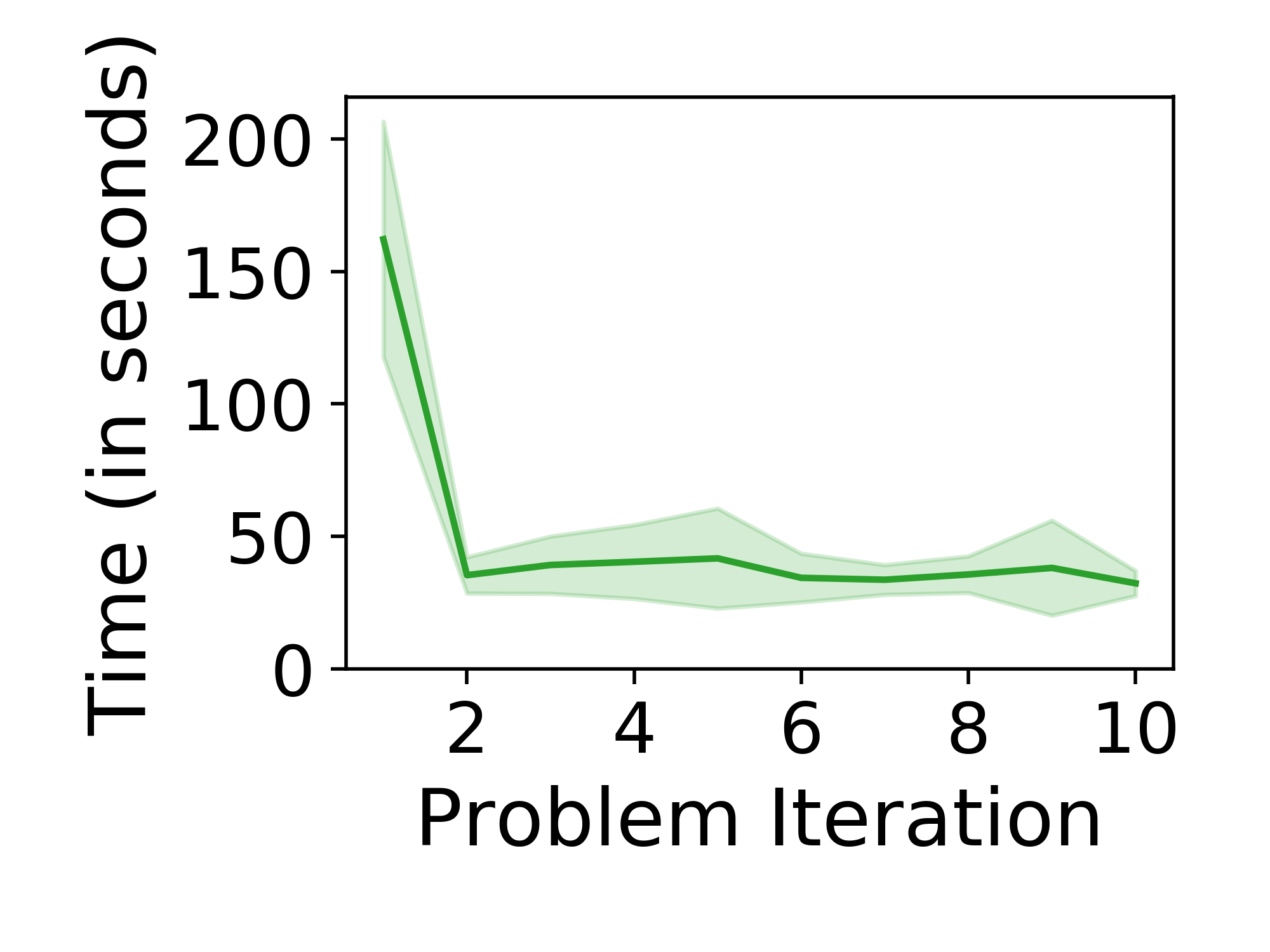}   \label{fig:experiments_A} }
    \subfigure[]{\includegraphics[width=0.48\columnwidth]{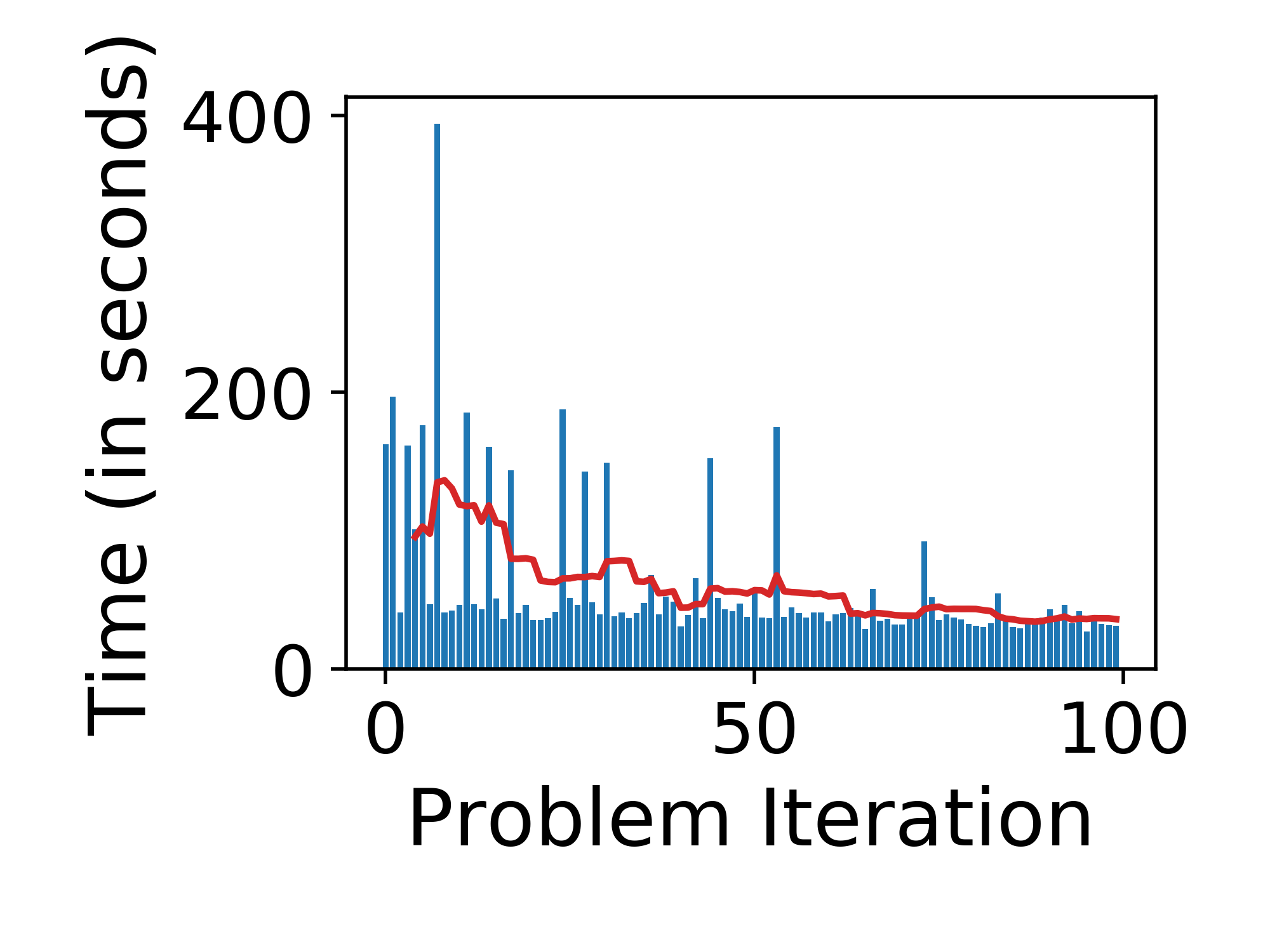}       \label{fig:experiments_B}}
    \caption{(a) Solving $\mathbf{20}$ randomly generated problems repeatedly $1\mathbf{10}$ times. x-axis show the problem iteration and y-axis shows the average time over $\mathbf{20}$ problem instances. (b) Time taken to solve $\mathbf{100}$ randomly generated problem instances. X-axis shows the problem number and y-axis shows the taken to solve each problem.}

    \Description{Plots to evaluate if dynamically updating the heurstic function imporve efficiency of our approach or not. }
    % \vspace{-2em}
\end{figure}

Our approach also outperforms the learning-based planner  \linebreak
LLP~\citep{dan_llp} (Fig. \ref{fig:results}), which uses learned critical regions but does not use state and action abstractions and does not perform hierarchical planning. This illustrates the value of learning abstractions and using them efficiently for hierarchical planning.

Similarly, we also evaluate our approach against TogglePRM~\citep{denny-wafr-2012}. TogglePRM is written in C++ and accepts only discrete \emph{SE$^2$} configuration space for a simple dot robot. We created discrete variants of environments shown in Fig. \ref{fig:envs}(A) and \ref{fig:envs}(B) with a total of $50176$ states and compare the total number of nodes sampled. For $100$ random trials, on an average our approach generated $631 \pm 278$ and $496 \pm 175$ states compared to TogglePRM which generated $4234 \pm 532$ and $19234 \pm 4345$ states for discrete variants of the environments shown in Fig. \ref{fig:envs}(A) and \ref{fig:envs}(B) respectively. Our approach was able to outperform TogglePRM since these environments do not have $\alpha$-$\epsilon$-separable passages~\citep{denny-wafr-2012}.

\subsubsection{\textbf{$c)$ Does Dynamically Updating Heurstic Function Improve Efficiency?}}

We carried out two sets of experiments. In the first set of experiments, we generated $20$ random motion planning problems and solved each problem repeatedly for $10$ times while updating the heuristic function. We maintained separate copies of high-level heuristic functions for each problem. Fig. \ref{fig:experiments_A} shows the results for this set of experiments in the environment $E$ (Fig. \ref{fig:envs}(E)) with the $4$-DOF hinged robot. The x-axis shows the planning iteration and the y-axis shows the average time over randomly generated $20$ problem instances. We can see how planning time reduces drastically once costs for abstract actions are updated.

In the second set of experiments, we generated $100$ random pairs of initial and goal states and computed motion plans for each of them. This time, we maintained a single heuristic function across all problems and updated it after each motion planning query. Fig.~\ref{fig:experiments_B} shows the result of the experiment in the environment $E$ (Fig. \ref{fig:envs}(E)) with the $4$-DOF hinged robot $H$. The x-axis shows the problem number and the y-axis shows the time taken by our approach to compute a solution. The red line in the plot shows the moving average of planning time. Fig.~\ref{fig:experiments_B} shwows  that dynamically updating the heuristic function for high-level planning helps to increase the efficiency of HARP and decrease motion planning times.

Empirical evaluation using these experiments validates our hypothesis that learning abstractions and effectively using them improves motion planning efficiency.

\section{Conclusion}
In this paper, we presented a probabilistically complete approach HARP, that uses deep learning to identify abstractions for the input configuration space. It learns state and action abstractions in a bottom-up fashion and uses them to perform efficient hierarchical robot planning. We developed a new multi-source bi-directional planning algorithm that uses learned state and action abstractions along with a custom dynamically maintained cost function to generate candidate high-level plans. A low-level motion planner refines these high-level plans into a trajectory that achieves the goal configuration from the initial configuration. 

Our formal framework provides a way to generate sound abstractions that satisfy the downward refinement property for holonomic robots. Our empirical evaluation on a large variety of problem settings shows that our approach is able to significantly outperform state-of-the-art sampling and learning-based motion planners. Through our empirical evaluation, we show that our approach is robust and can be scaled to large environments and to robots that have high degrees of freedom. Our work presents a foundation for learning high-level, abstractions from low-level trajectories. Currently, our approach works for deterministic robot planning problems. We aim to extend our approach to support stochastic settings and learn abstractions for task and motion planning problems.

\begin{acks}
We thank Abhyudaya Srinet for his help in implementing a primitive version of the presented work. We thank Kyle Atkinson for his help with creating test environments. This work is supported in part by the NSF under grants 1909370 and 1942856.
\end{acks}

%%%%%%%%%%%%%%%%%%%%%%%%%%%%%%%%%%%%%%%%%%%%%%%%%%%%%%%%%%%%%%%%%%%%%%%%

%%% The next two lines define, first, the bibliography style to be 
%%% applied, and, second, the bibliography file to be used.

% \balance
\bibliographystyle{ACM-Reference-Format} 

\bibliography{aaai22}

\clearpage
% \appendix
\onecolumn
\section*{APPENDICES}
% \addcontentsline{toc}{section}{Appendices}
\renewcommand{\thesection}{\Alph{section}}
\setcounter{section}{0}

\section{Beam Search}

\begin{algorithm}[h!]
\caption{Beam Search}
\label{alg:beam_search}
\KwIn{Graph $G = <V, E>$ , initial state $s_0 \in V$, goal state $s_g \in V$, beam width $w$, heuristic $h$}
\KwOut{A set of paths from $s_0$ to $s_g$}

fringe = PriorityQueue() \\ 
fringe.add(0,($s_0$,None,[]))\\
\While{path not found  and fringe is not empty}{
    working\_fringe $\gets$ select top $w$ nodes from the fringe \\
    empty\_fringe(fringe) \\ 
    \While{working\_fringe is not empty}{
        current, path, visited $\gets$ working\_fring.pop() \\
        \If{current = $s_g$}{
            return path \\
        }
        \Else{
            path.add(current) \\ 
            visited.add(current) \\ 
            \ForEach{ node $\in $ current.successors }{
                \If{node $\notin$ visited }{
                    $p \gets g$(current) $+$ $h(s')$ \\ 
                    fringe.add($p$, (current, path, visited))                 }
            }
        }
    }
}
return false
\end{algorithm}
% ~\newpage

\section{Networks Architecture}

\begin{figure}[h!]
    \centering
    \includegraphics[width=\textwidth]{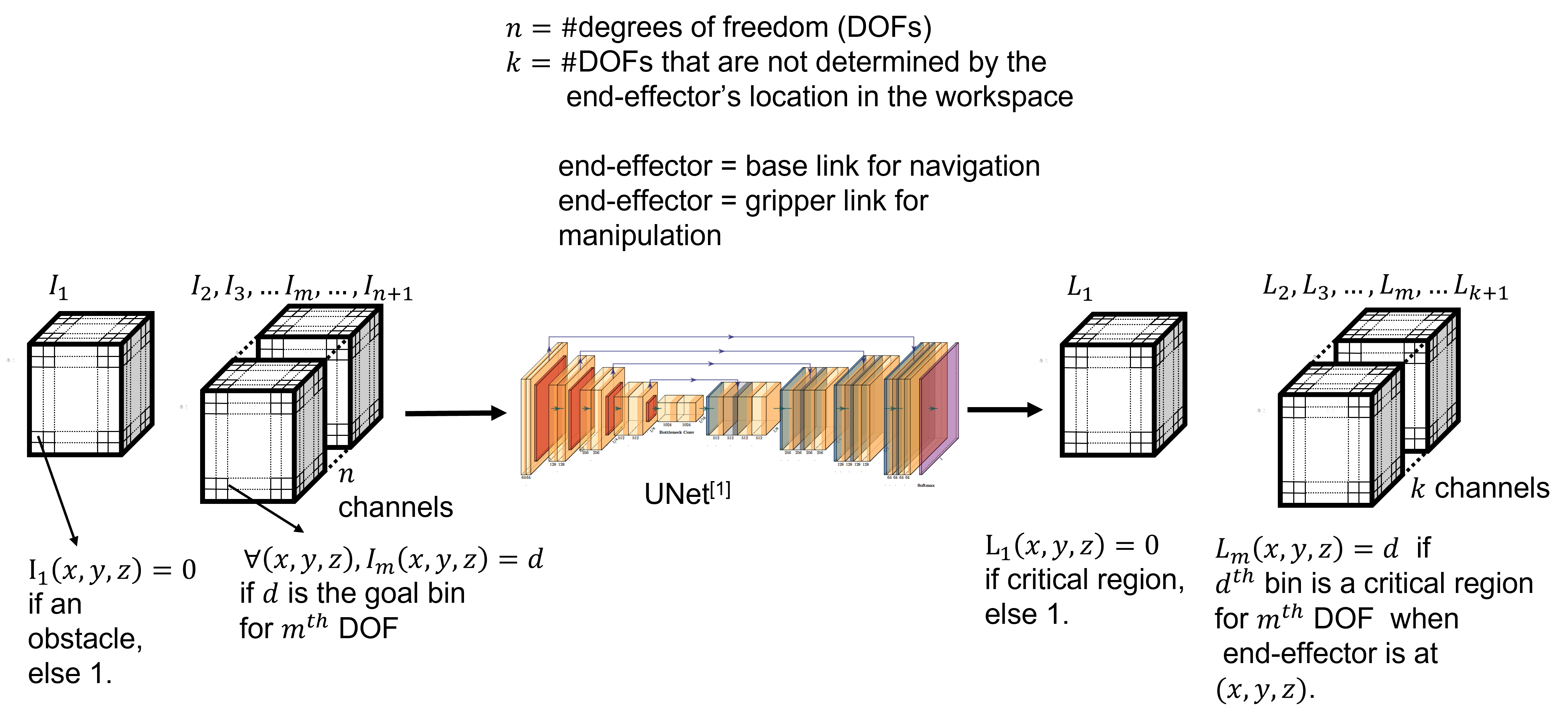}
    \label{fig:nn_arch}
\end{figure}

% \blfootnote{$^{[1]}$ O. Ronneberger, P.Fischer, and T. Brox. 2015. U-Net: Convolutional Networks for Biomedical Image Segmentation. In Proc. MICCAI, 2015}
\blfootnote{[1] O. Ronneberger, P.Fischer, and T. Brox. 2015. U-Net: Convolutional Networks for Biomedical Image Segmentation. In Proc. MICCAI, 2015}

~\newpage
\section{Results}
\subsection{$3$-DOF Rectangular and Car Robots}

% \begin{enumerate}
\begin{center}

\begin{figure}[h!]
    \setkeys{Gin}{width=\linewidth}
    \begin{center}
    \centering
\begin{tabularx}{\textwidth}{p{0.041\textwidth} p{0.19\textwidth} p{0.19\textwidth} p{0.19\textwidth} p{0.19\textwidth} p{0.04\textwidth} }        &\includegraphics{env_8_0_1.png} &  \includegraphics{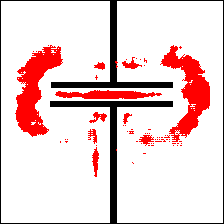} &  \includegraphics{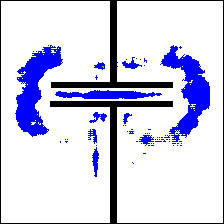} & \includegraphics{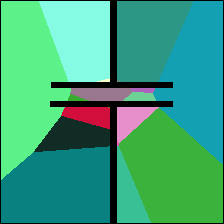} &\\
        & \includegraphics{env_10_2_1.png} &  \includegraphics{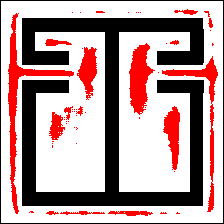} & \includegraphics{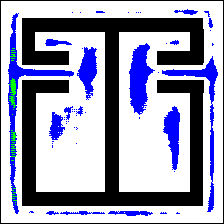} & \includegraphics{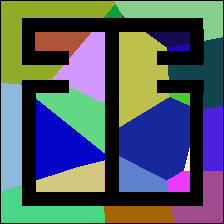} &\\
         & \includegraphics{env_41_0_1.png} &  \includegraphics{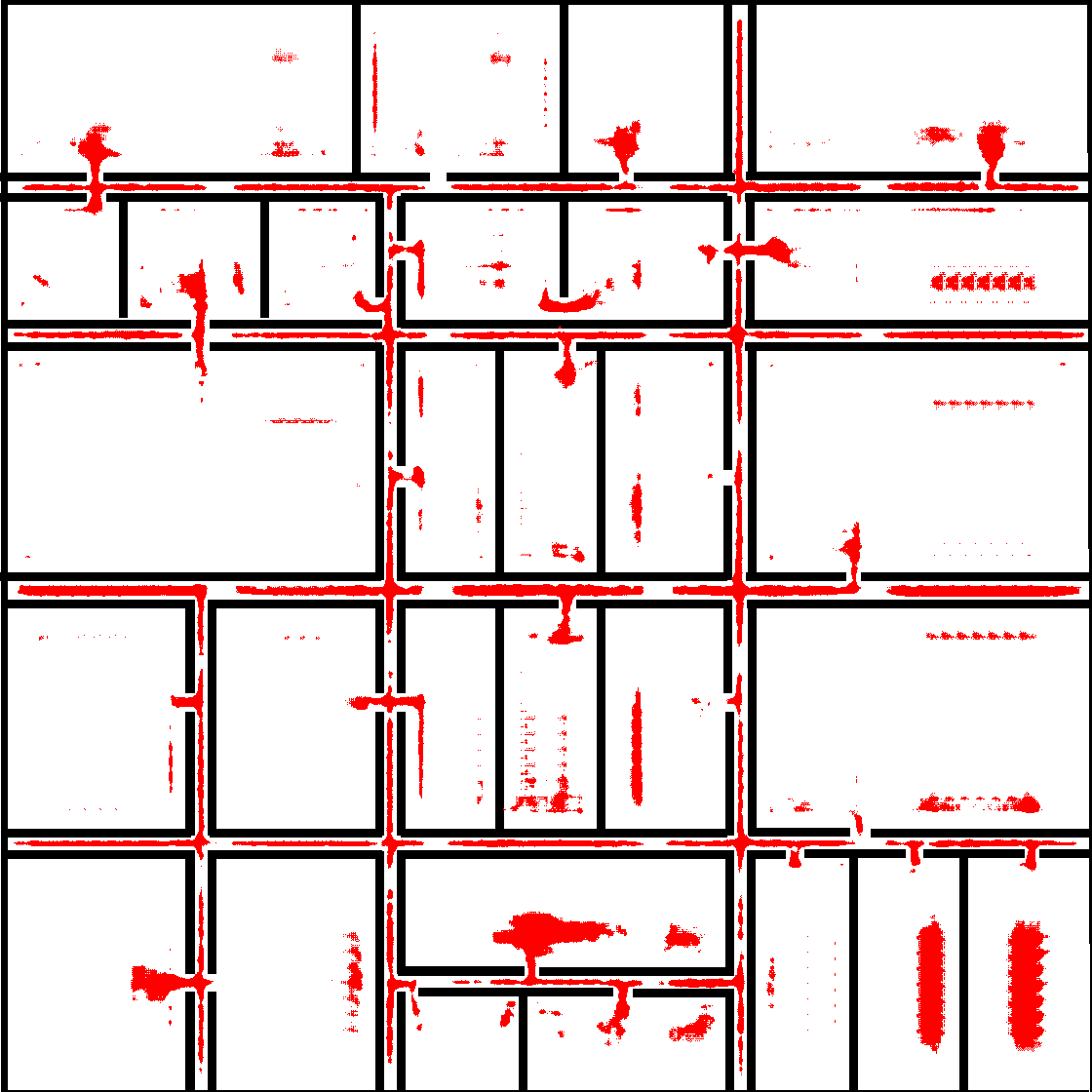} &  \includegraphics{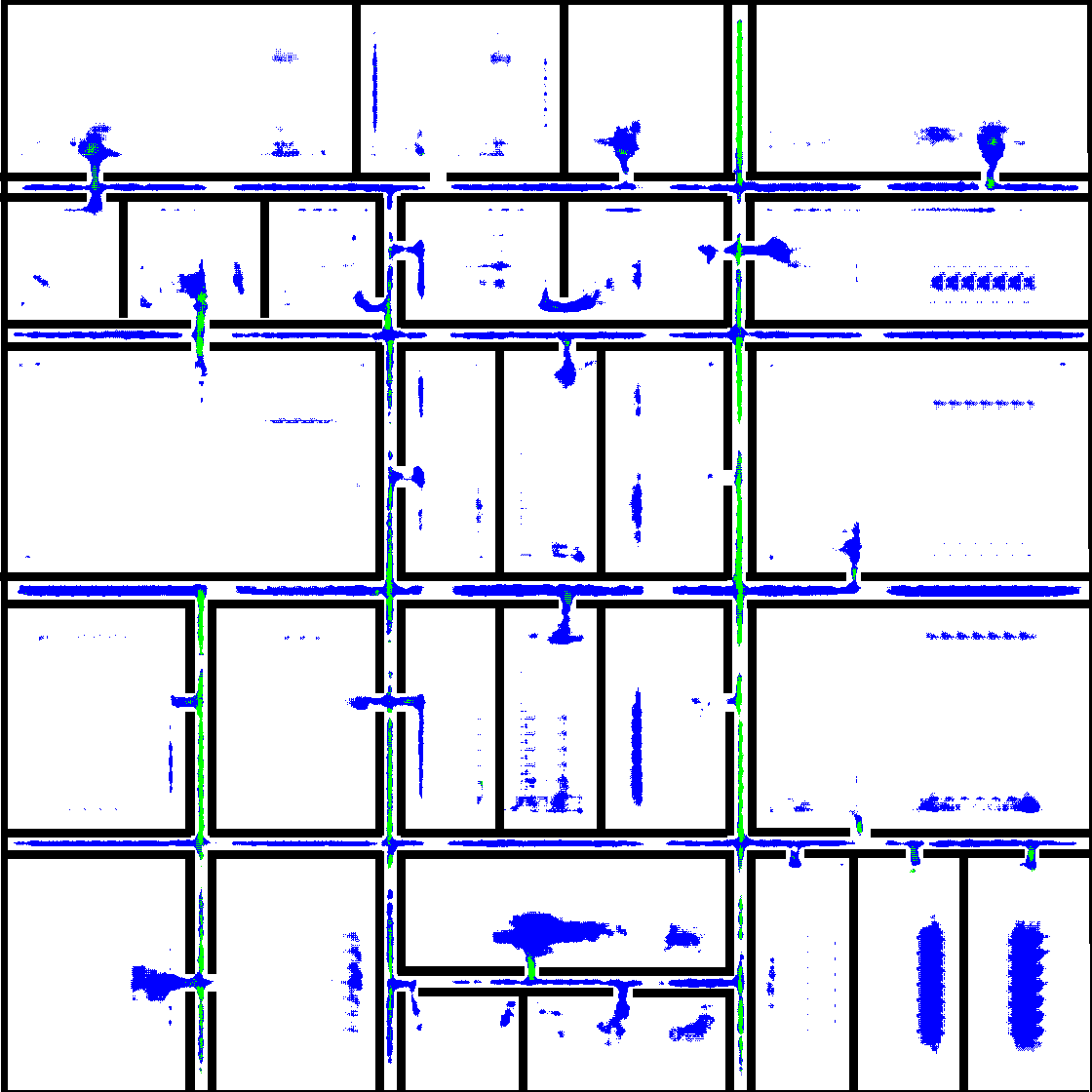} & \includegraphics{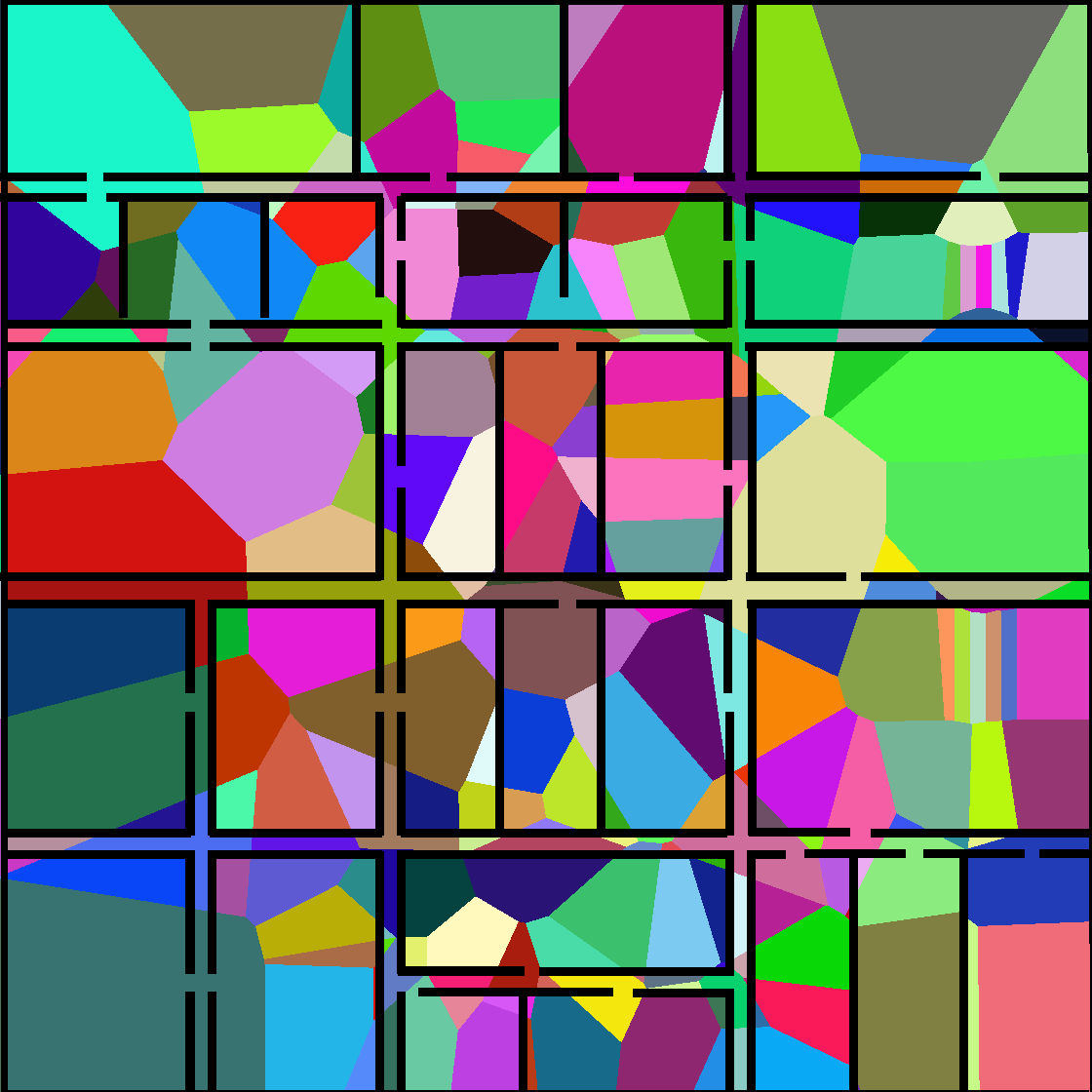} &\\
        & \includegraphics{env_42_0_1.png} &  \includegraphics{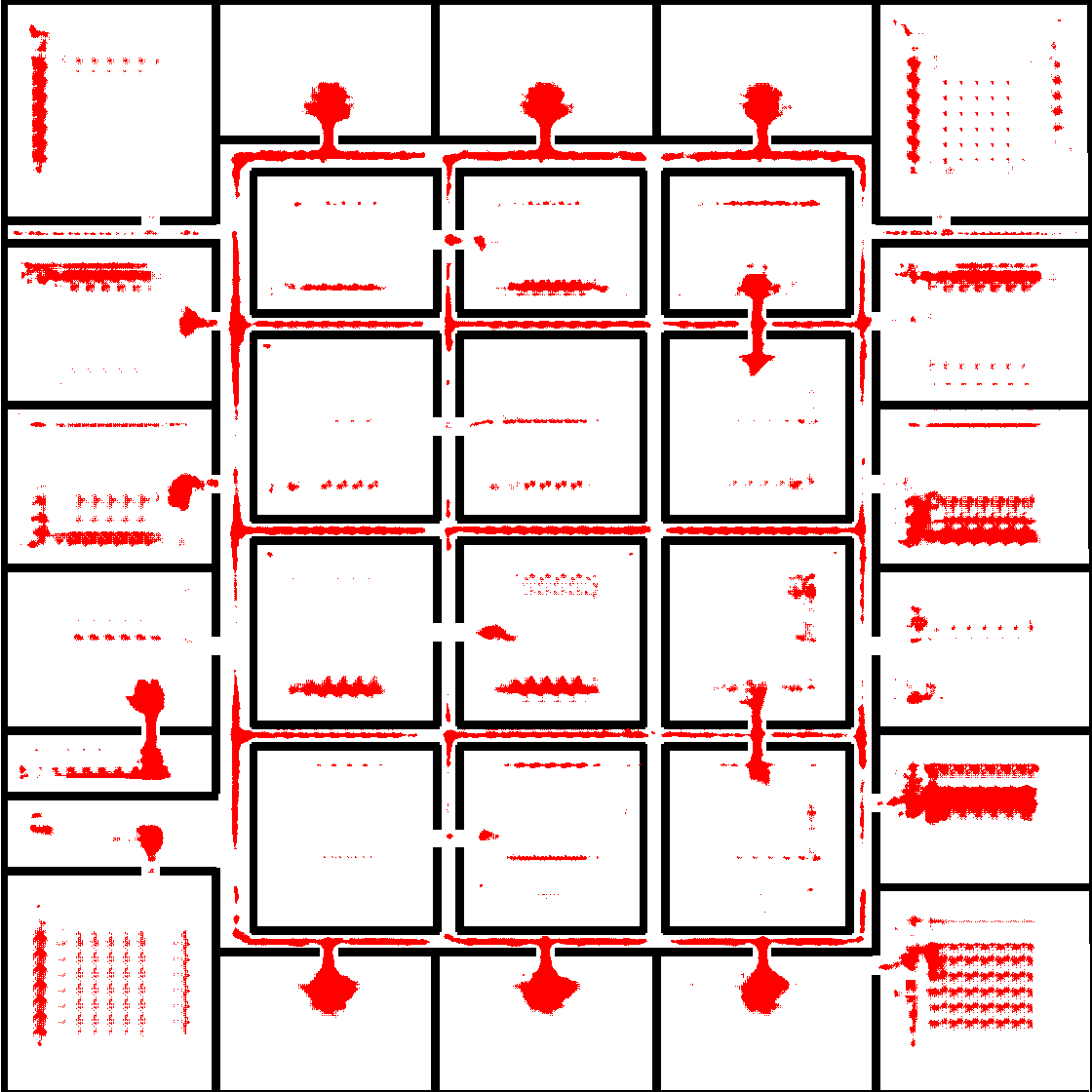} & \includegraphics{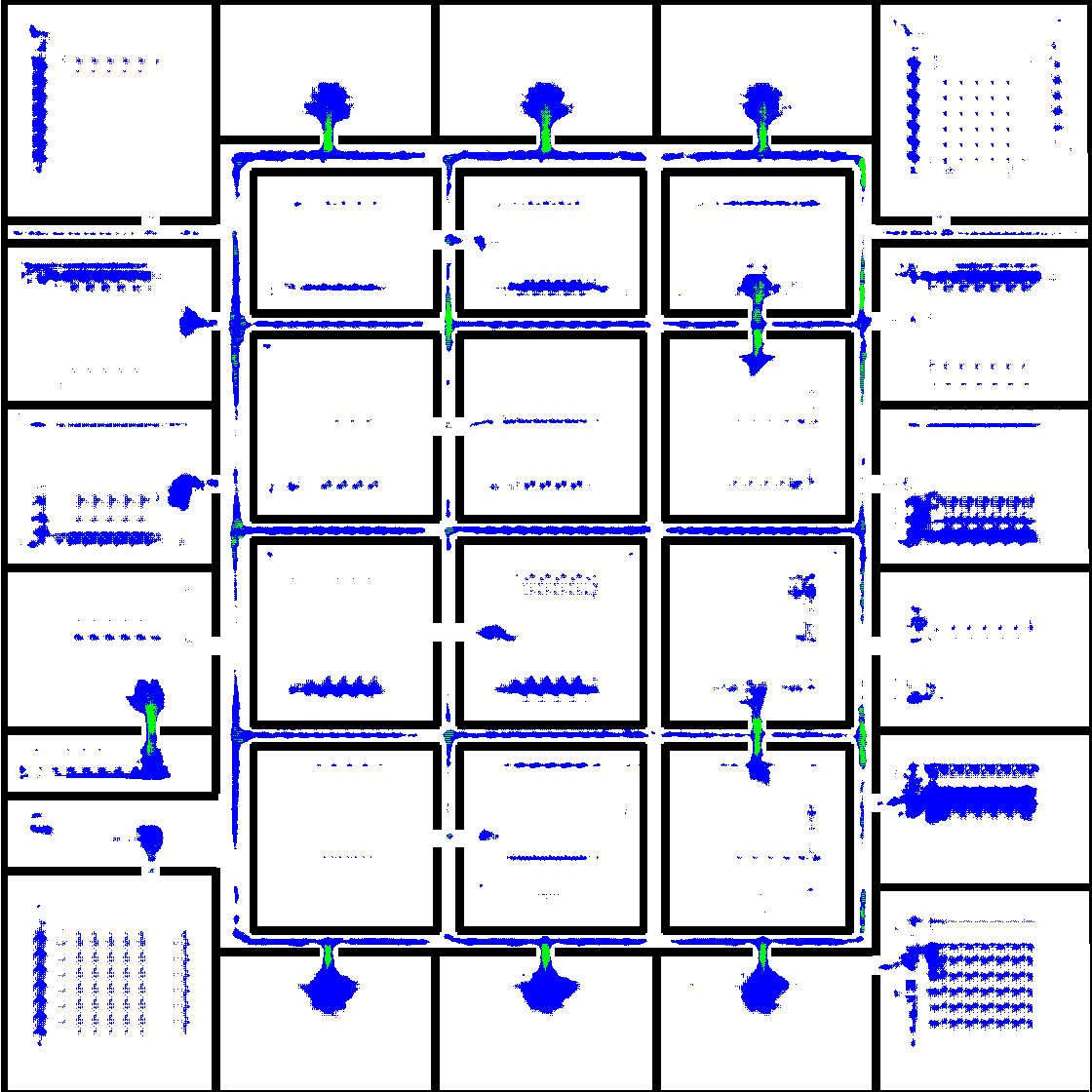} & \includegraphics{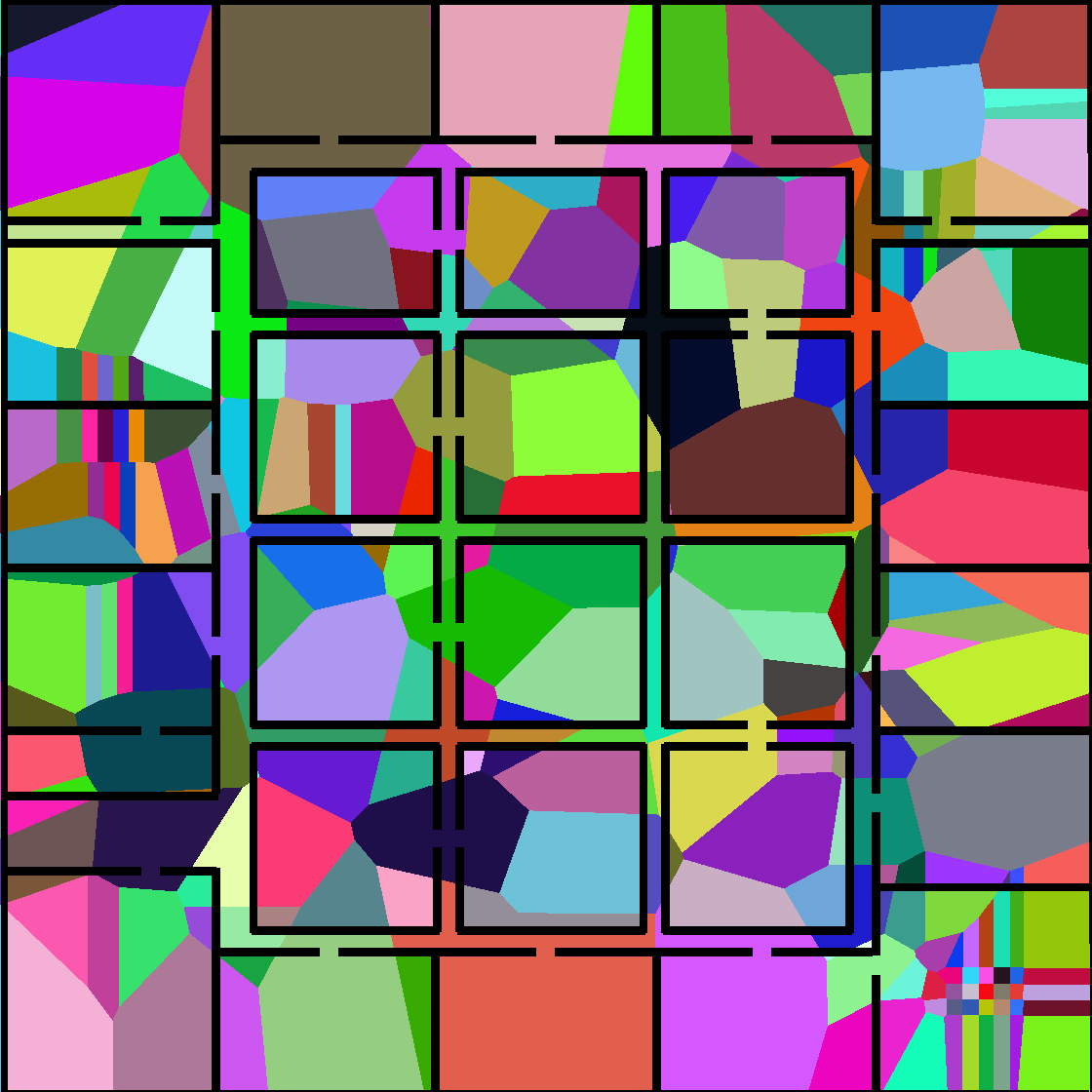} &\\
        & \includegraphics{env_43_0_1.png} &  \includegraphics{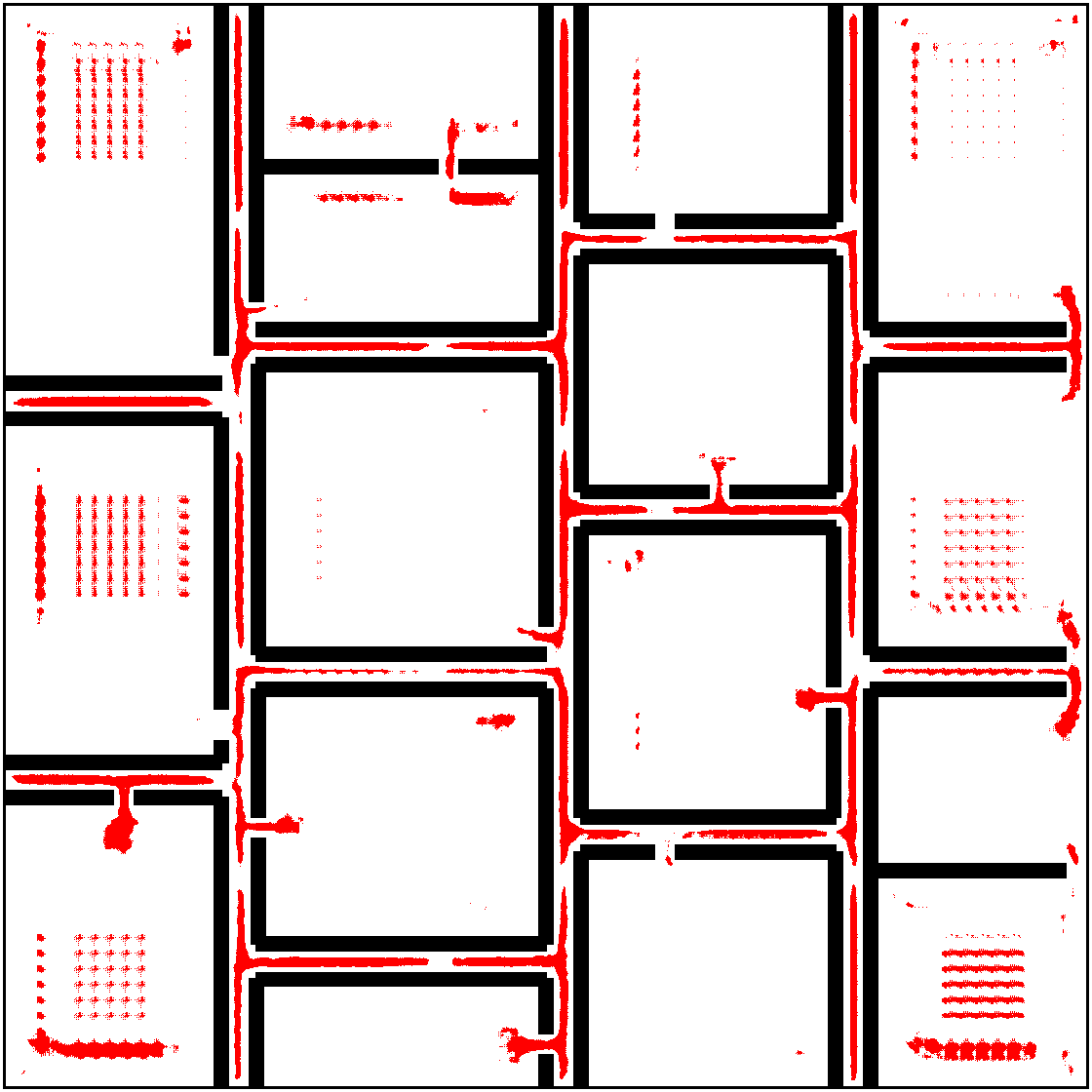} & \includegraphics{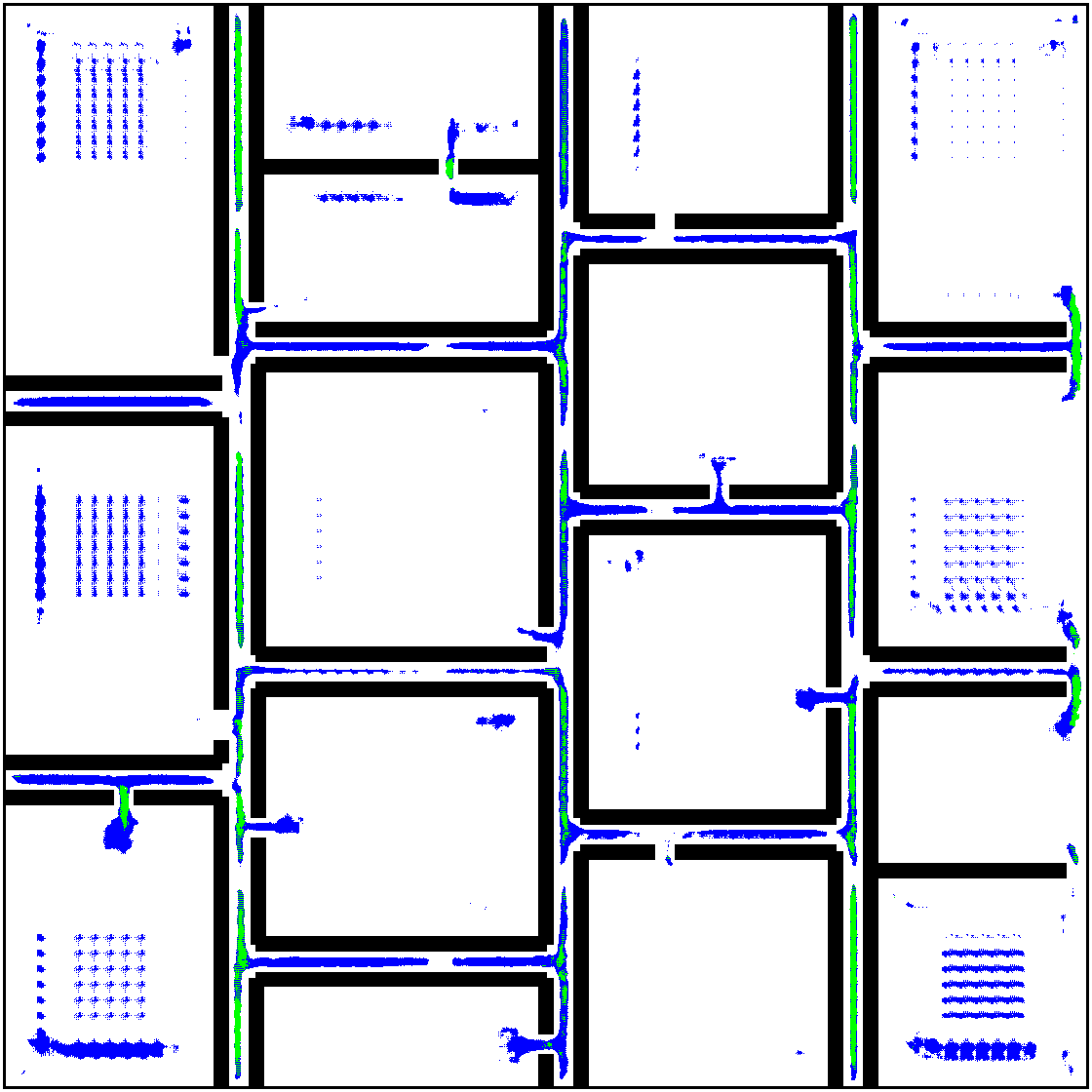} & \includegraphics{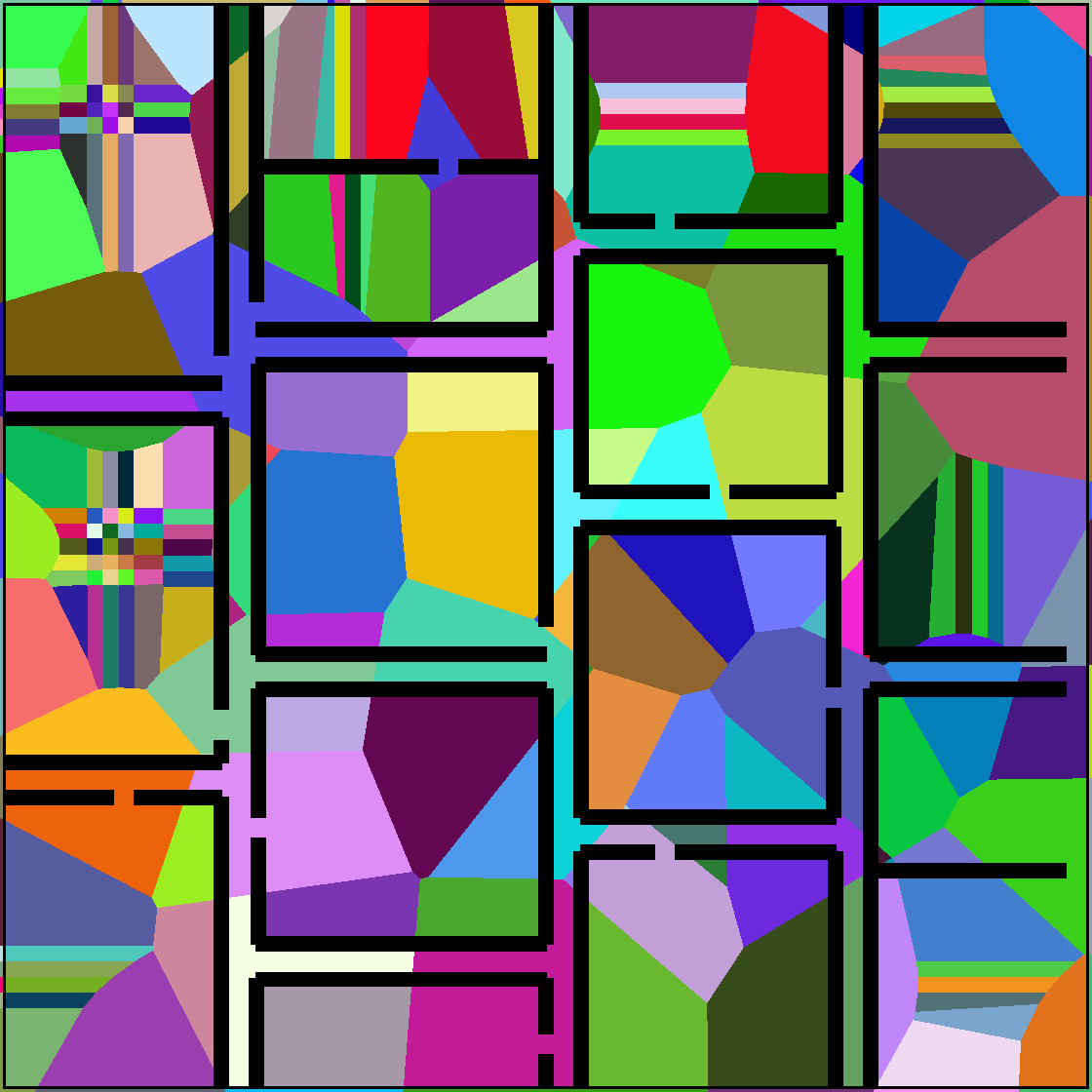} &\\
        & \centering (a) & \centering (b) & \centering (c) & \centering(d) & \\
    \end{tabularx}
    \end{center}
    \caption{Predicted critical regions and computed state abstractions for $3$-DOF rectangular (R) and Car robots. (a) Input to the environment. (b) Critical regions for the location of the robot's base link in the workspace. (c) Critical regions for the  rotation of the robot's base 
    link of the the robot. Blue regions are locations where the network predicted the robot to be horizontal and green regions are the regions where the network predicted the robot to be vertical.
    (d) 2D projections of state abstraction generated by our approach. State abstractions are strictly for visualization as our approach does not require to generate them.
    }
    % \label{fig:test_3dof}

    % \vspace{-1em}
\end{figure}

\end{center}

% \end{enumerate}
\newpage
\subsection{$4$-DOF Hinged Robot}

\begin{figure}[h!]
    \setkeys{Gin}{width=\linewidth}
    \begin{center}
    \begin{tabularx}{\textwidth}{p{0.18\textwidth} p{0.18\textwidth} p{0.18\textwidth} p{0.18 \textwidth} p{0.18\textwidth} }
         \includegraphics{env_8_0_1.png} &  \includegraphics{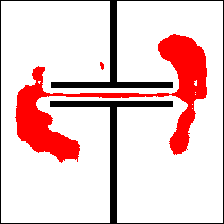} &  \includegraphics{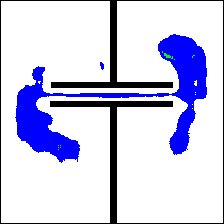} & \includegraphics{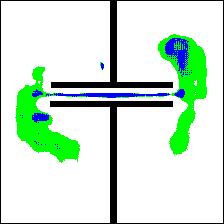} &  \includegraphics{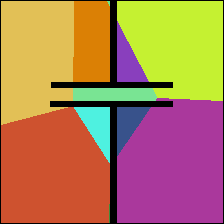} \\
         \includegraphics{env_10_2_1.png} &  \includegraphics{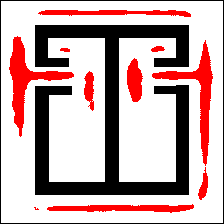} & \includegraphics{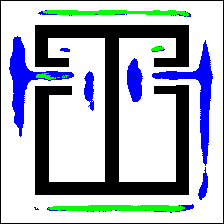} & \includegraphics{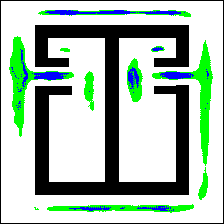} & \includegraphics{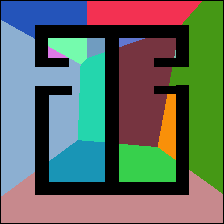}\\
          \includegraphics{env_41_0_1.png} &  \includegraphics{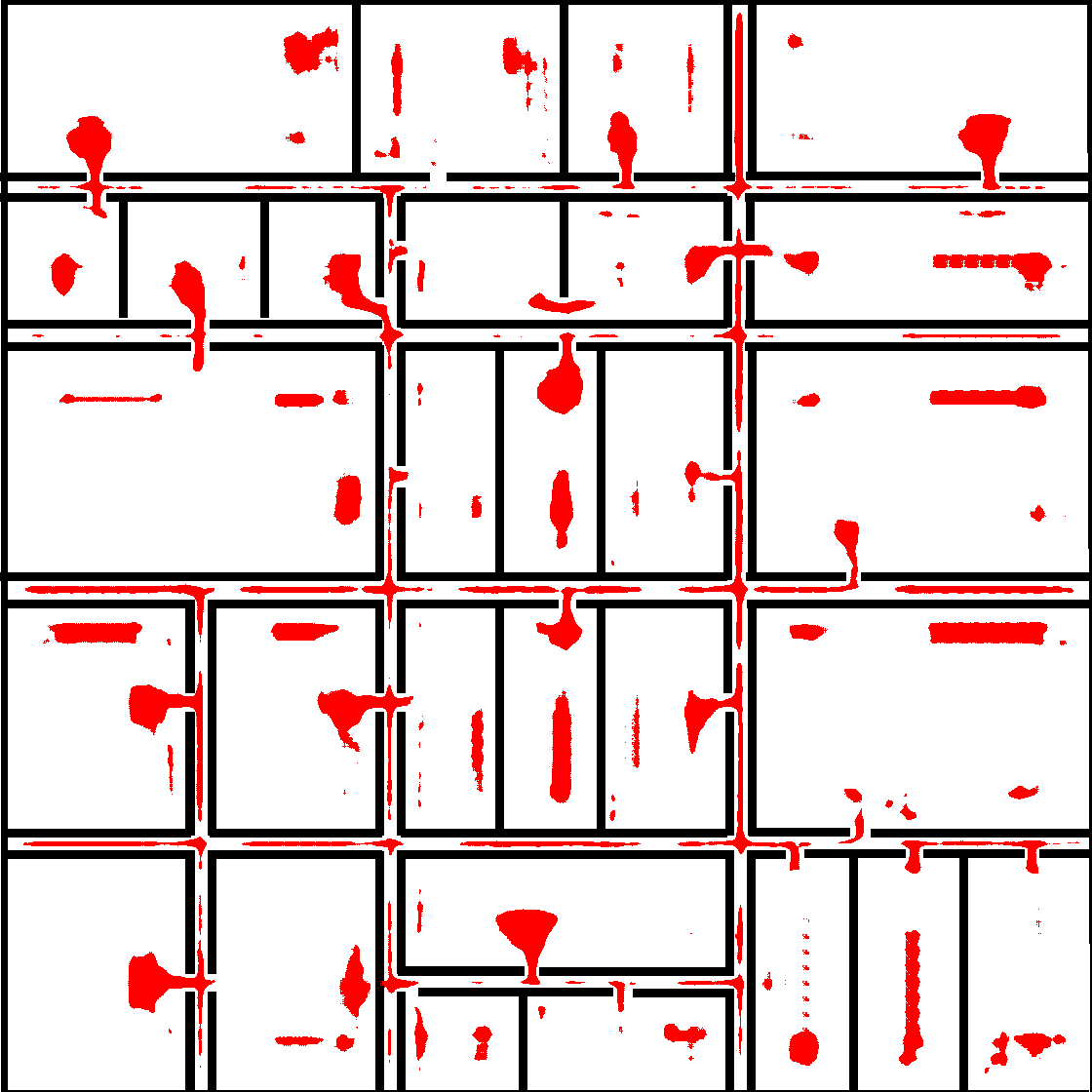} &  \includegraphics{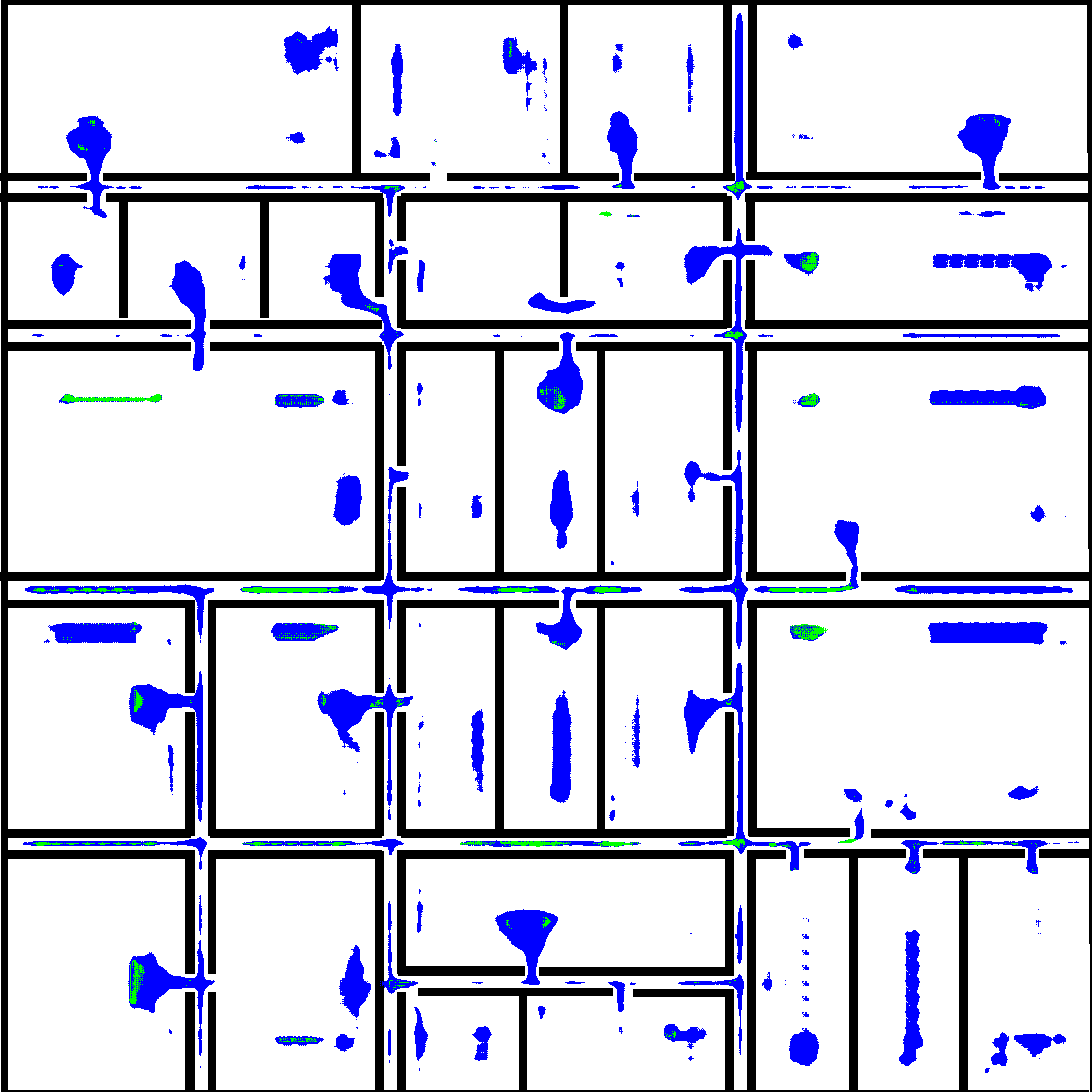} & \includegraphics{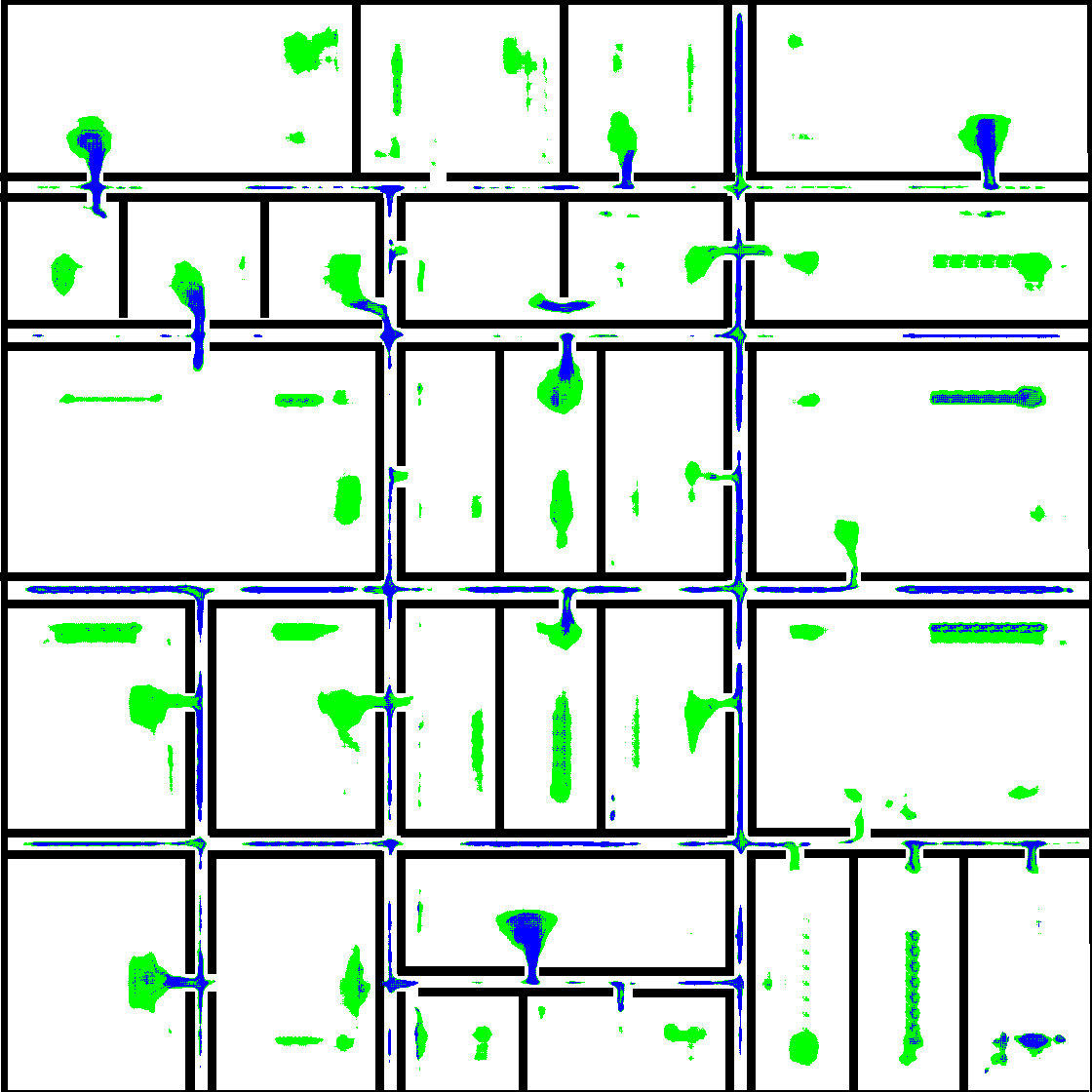} &\includegraphics{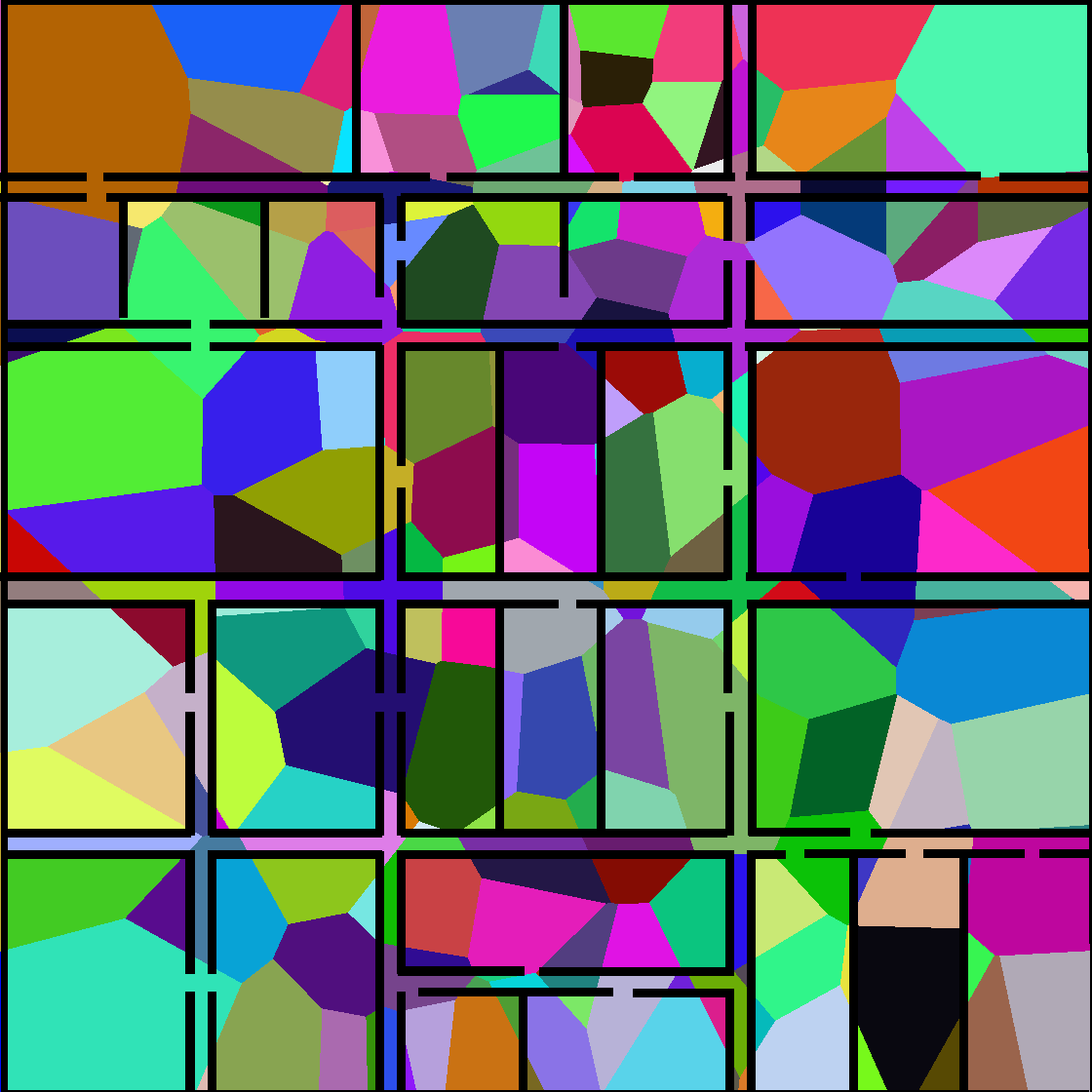}\\
          \includegraphics{env_42_0_1.png} &  \includegraphics{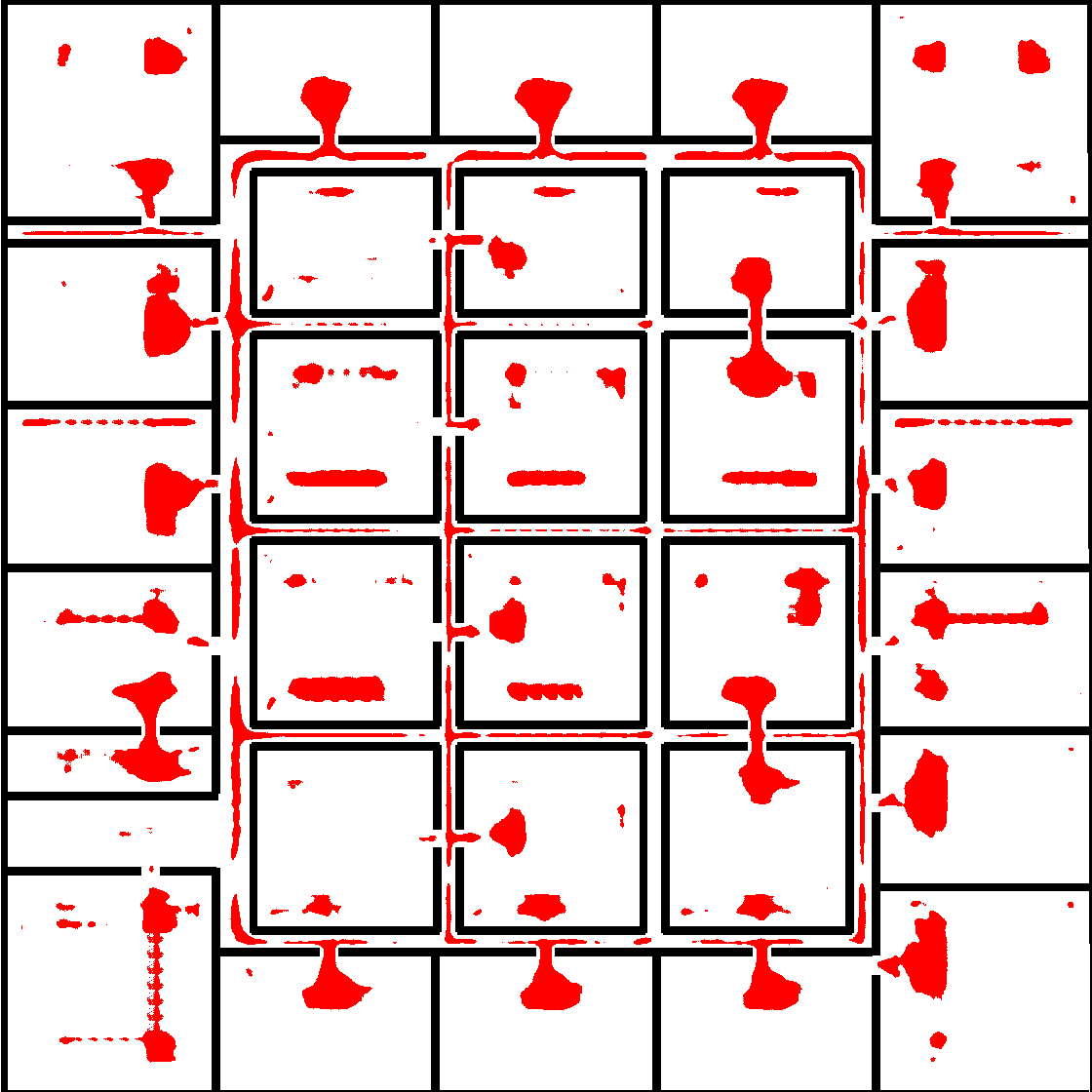} &  \includegraphics{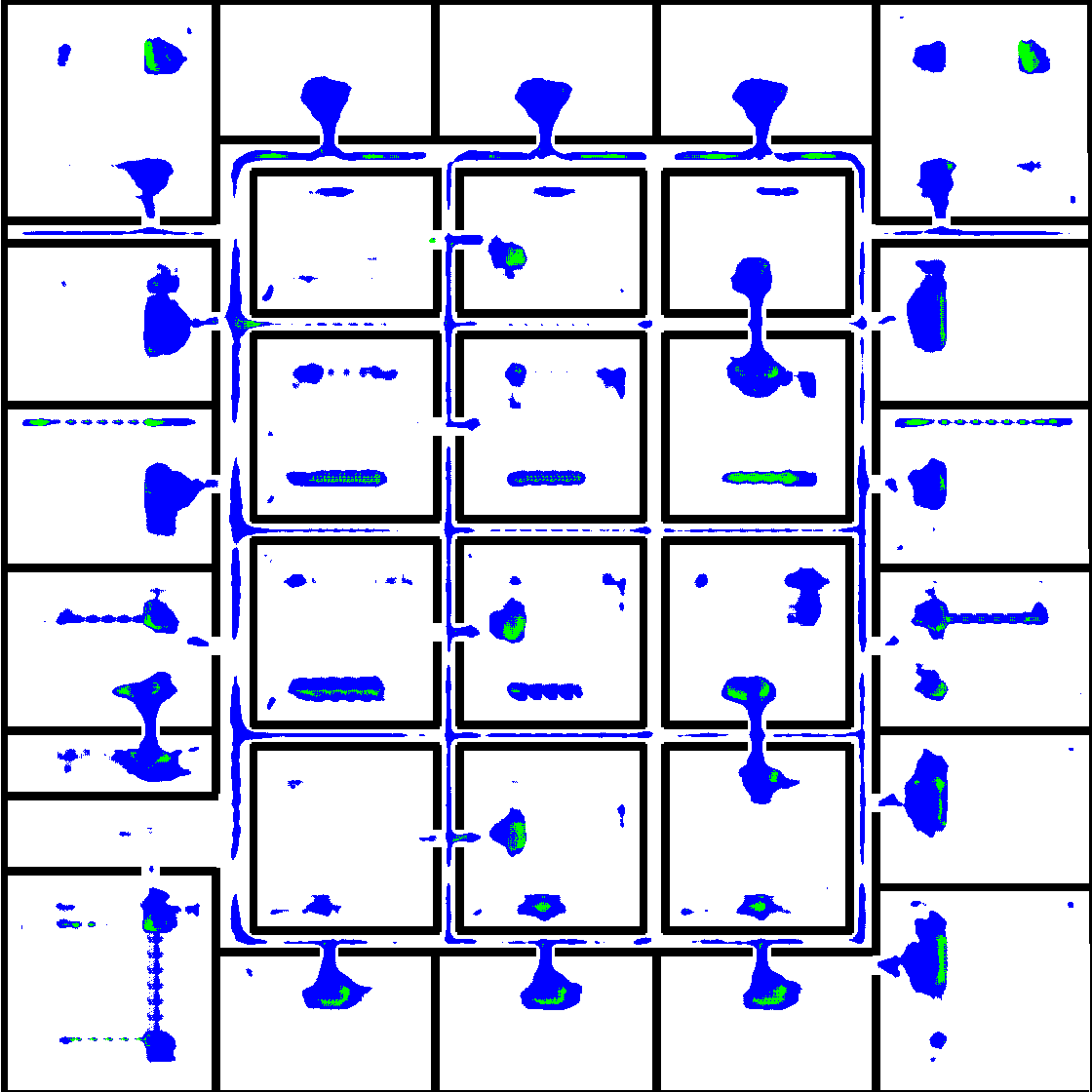} & \includegraphics{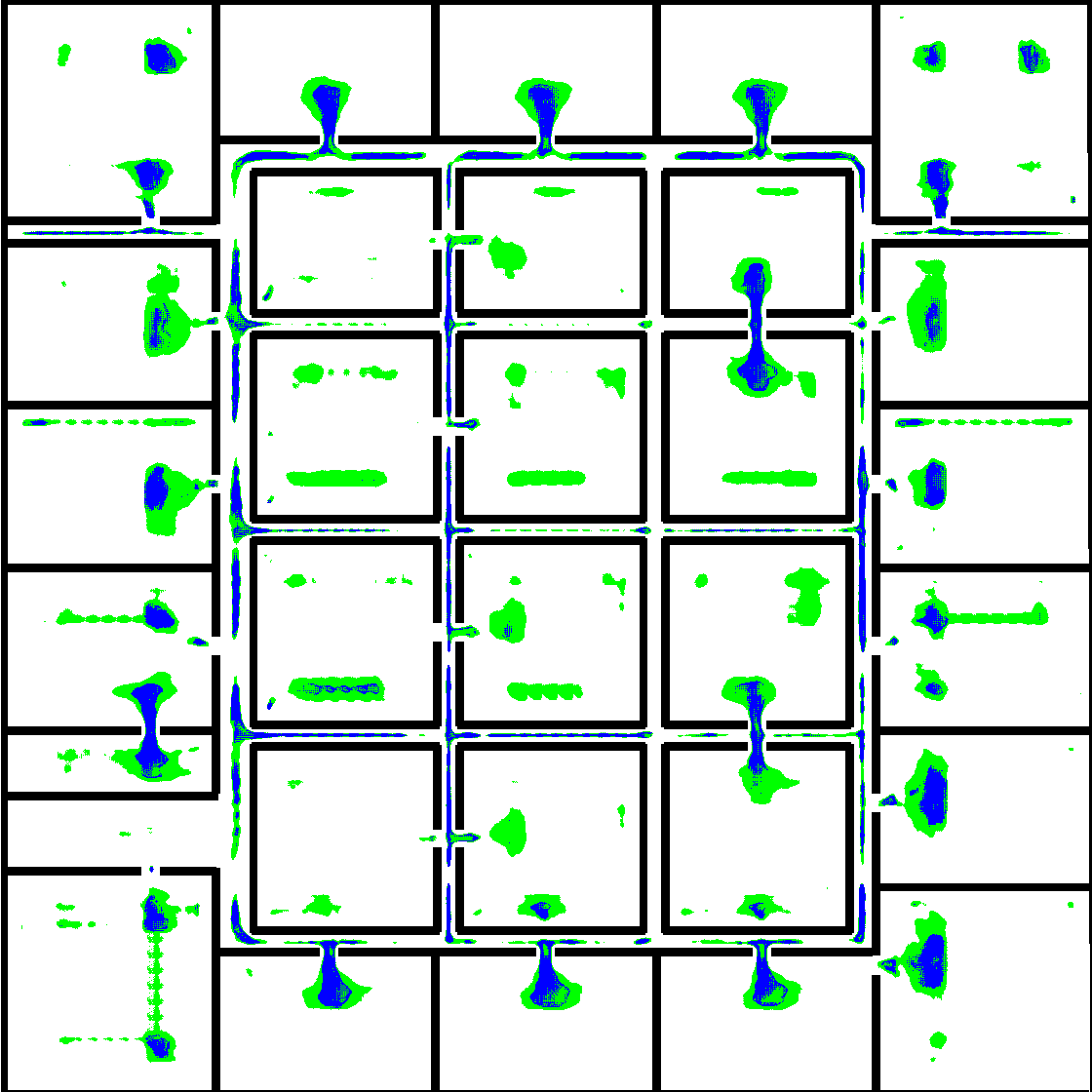} &\includegraphics{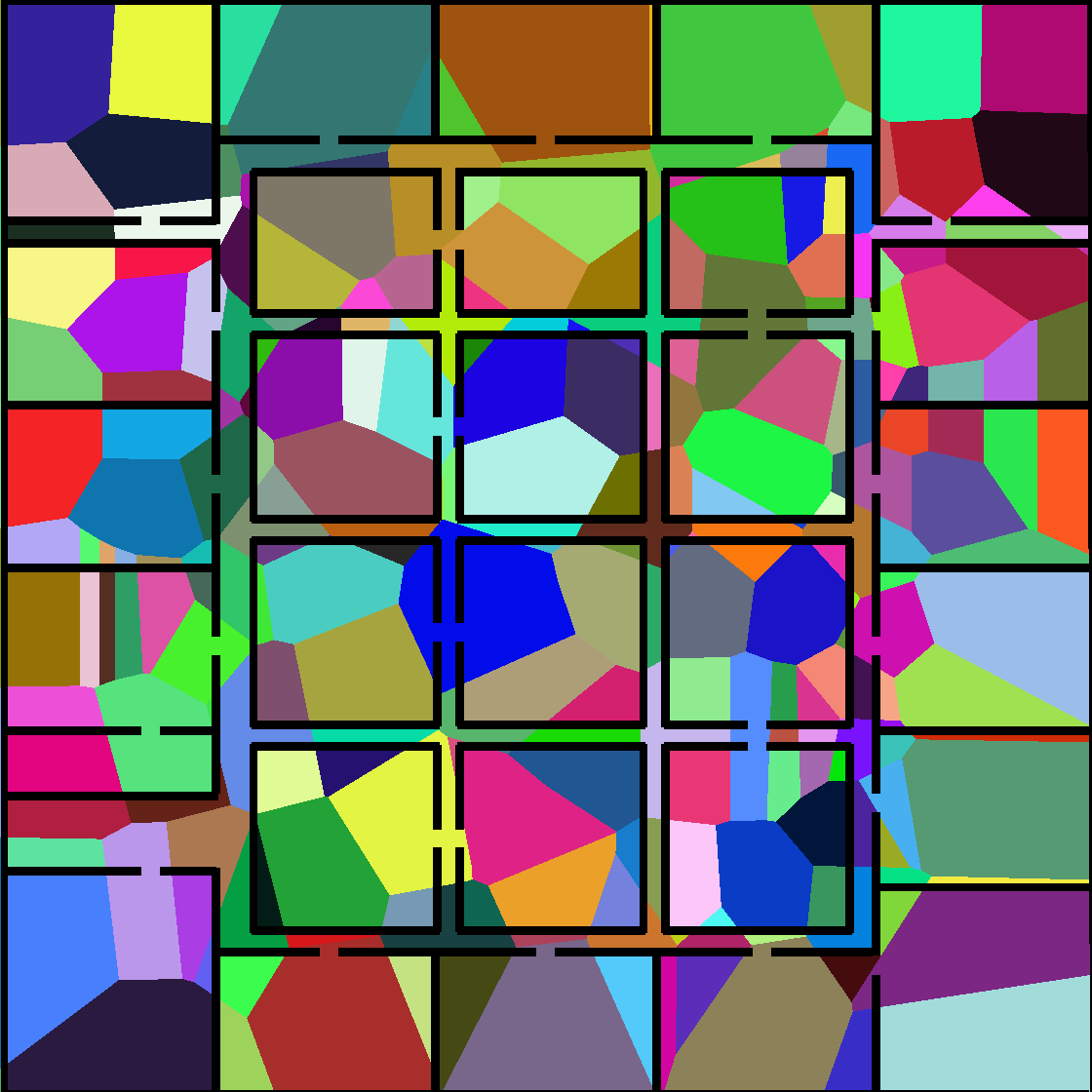}\\
         \includegraphics{env_43_0_1.png} &  \includegraphics{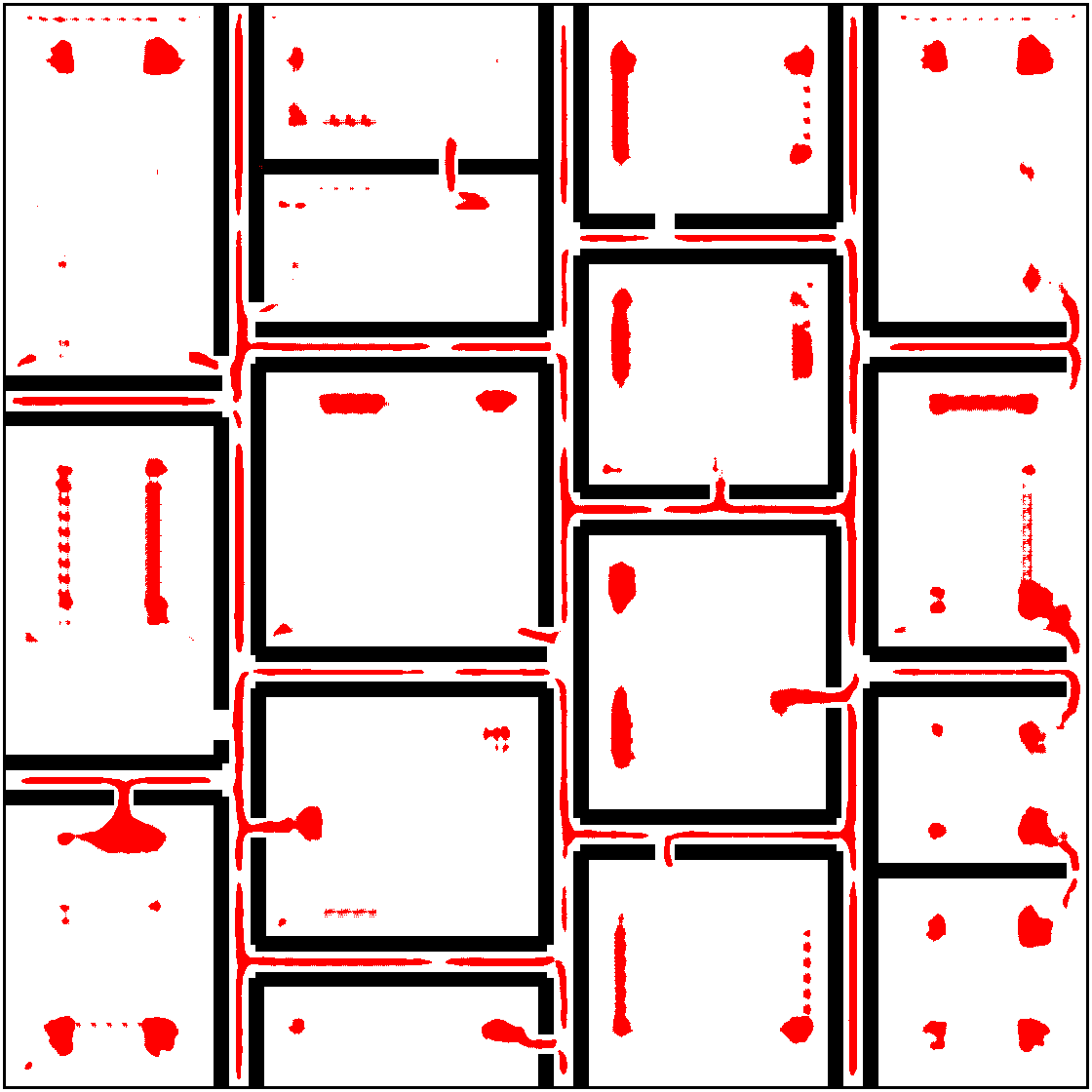} &  \includegraphics{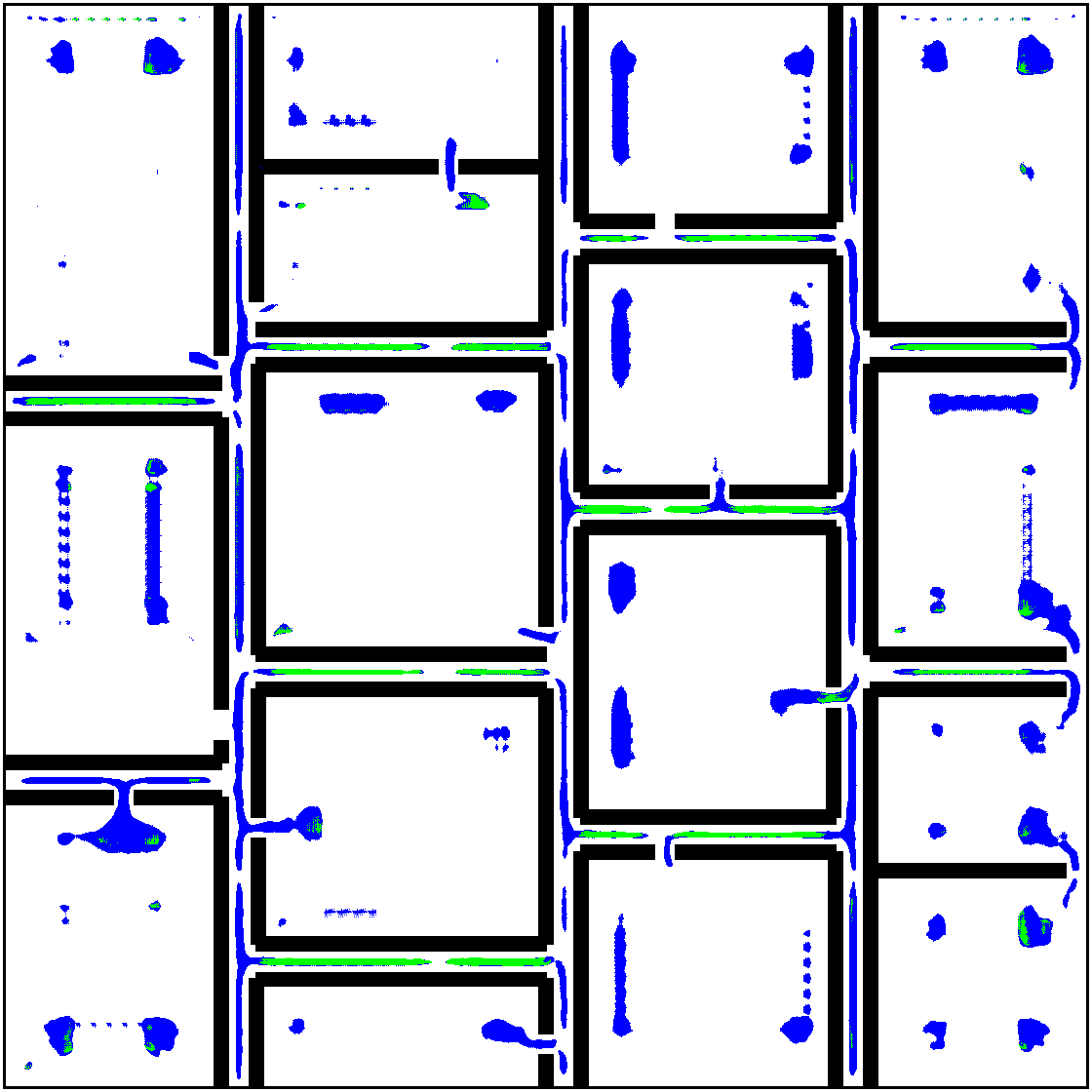} & \includegraphics{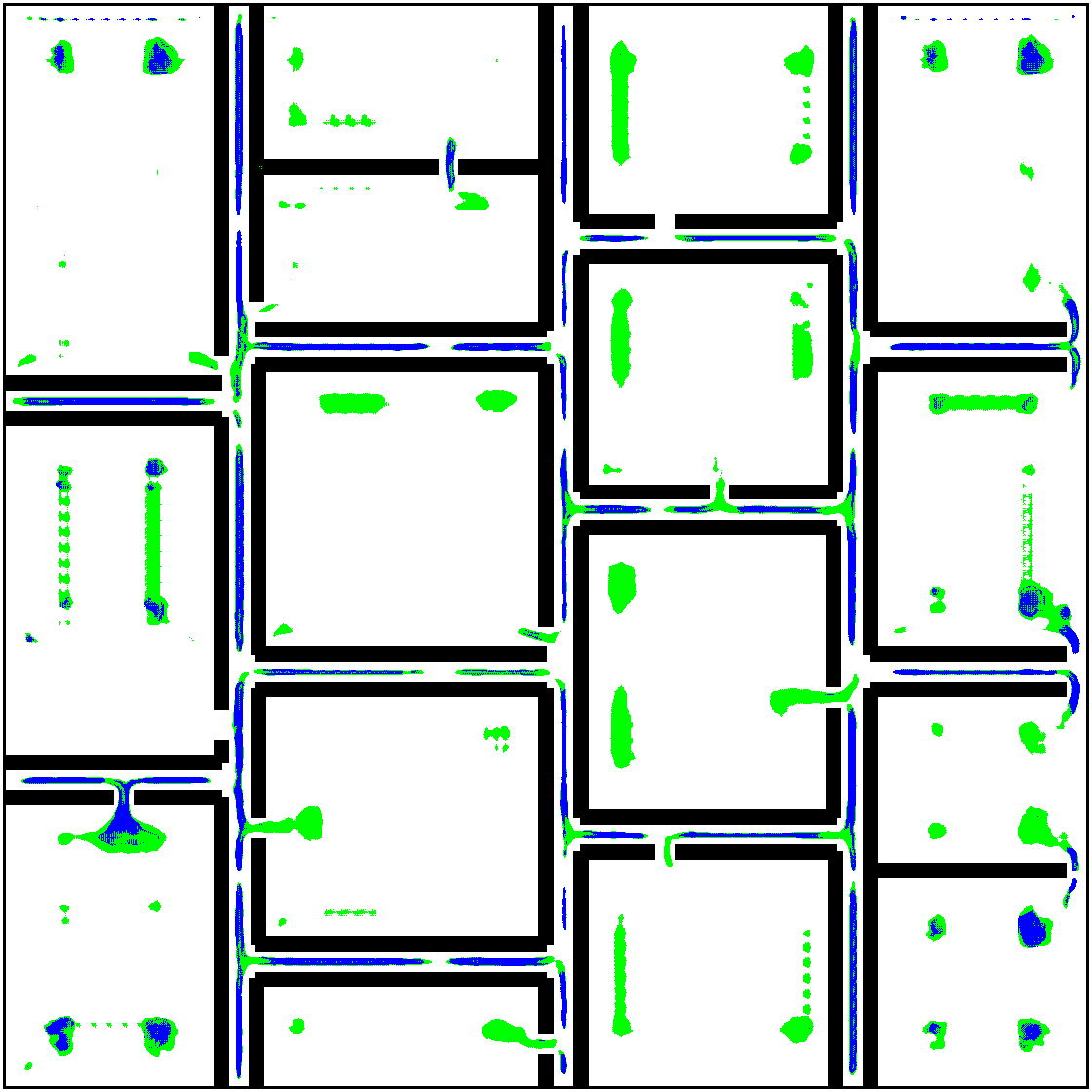} &\includegraphics{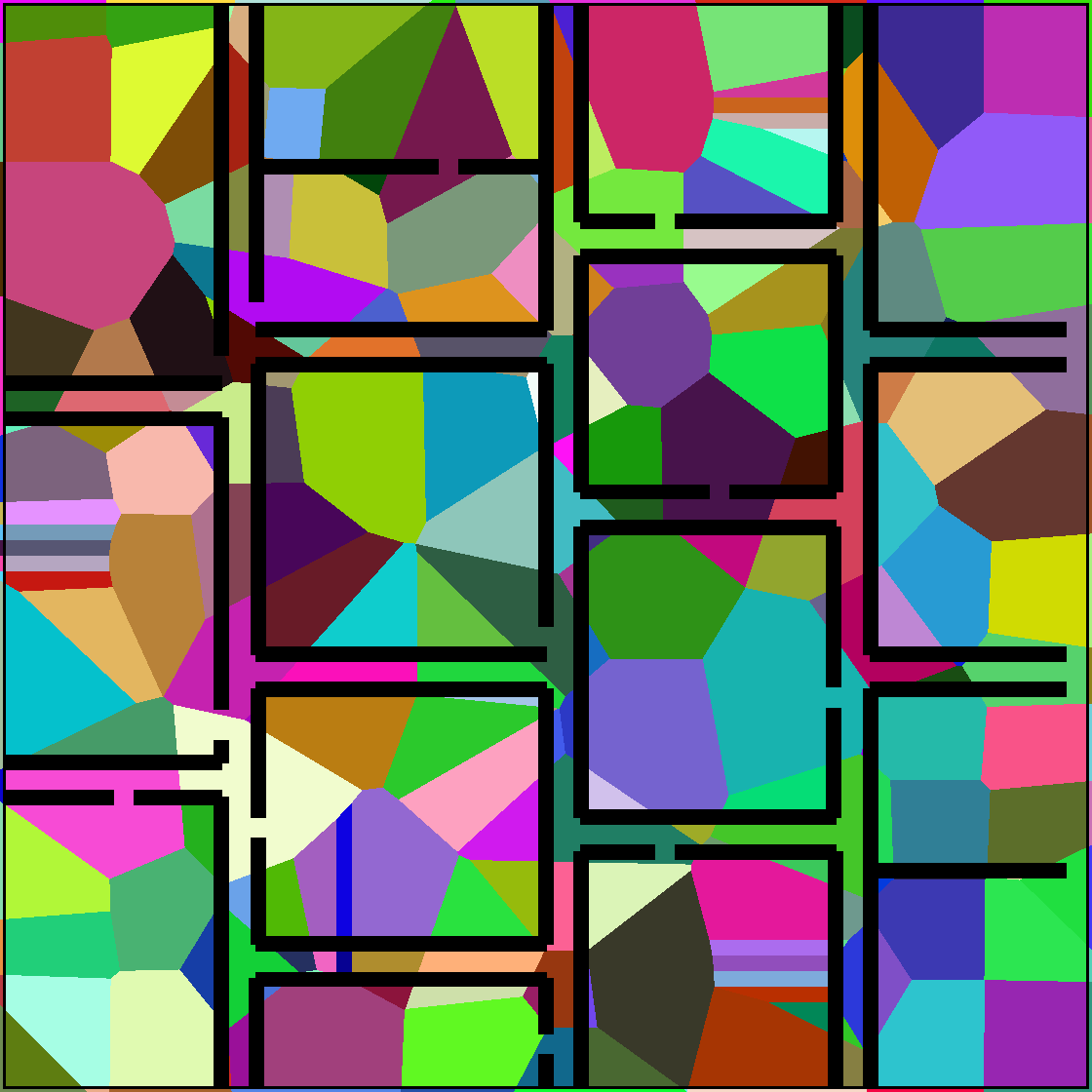}\\
        \centering (a) & \centering (b) & \centering (c) & \centering(d) & \centering(e)  \\
    \end{tabularx}
    \end{center}
    \caption{ Predicted critical regions and generated state abstractions for a $4$-DOF hinged robot. (a) Input to the environment. (b) Critical regions for the location of the robot's base link in the workspace. (c) Critical regions for the  orientation of  the the robot's base link. Blue regions are locations where the network predicted the robot to be horizontal and green regions are the regions where the network predicted the robot to be vertical. (d) Critical regions  for the hinge joint. Blue regions show that the network predicted it to be closer to $180^{\circ}$ and green regions show that the network predicted the hinge angle close to $90^{\circ}$ or $270^{\circ}$. (e) 2D projections of state abstraction generated by our approach. State abstractions are strictly for visualization as our approach does not require to generate this explicitly.}
    \label{fig:test_3dof}

    % \vspace{-1em}
\end{figure}

\subsection{$8$-DOF Fetch Robot}

\begin{figure}[h!]
 \vspace{0.3em}
    \setkeys{Gin}{width=\linewidth}
    \begin{center}
    \begin{tabularx}{0.5\textwidth}{X|X}
        \includegraphics{fetch_test1.png} &  \includegraphics{fetch_test1_2.png}   \\
        \includegraphics{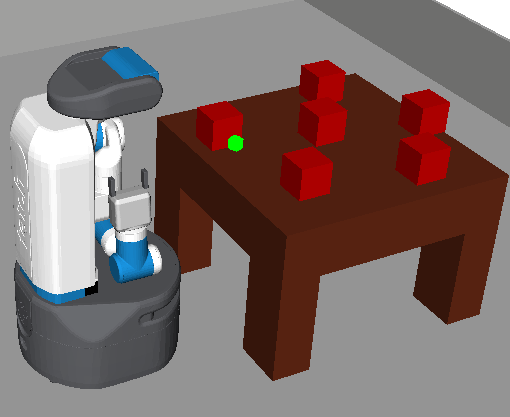} &    \includegraphics{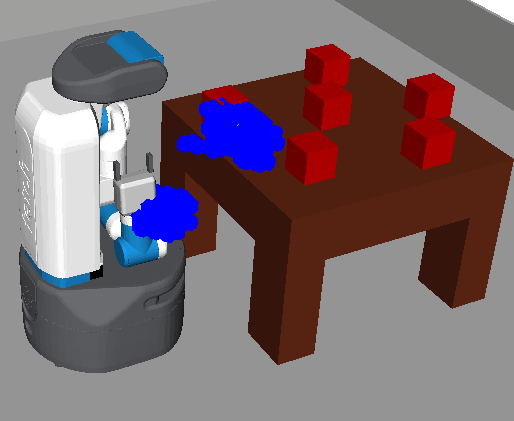}\\
        \centering (a) & \centering (b)  \\
    \end{tabularx}
    \end{center}
    \caption{ Predicted critical regions for an $8$-DOF Fetch robot. The green region in (a) shows the goal location for the end effector. (b) shows the critical regions generated by the learned model. Although the network predicts critical regions  for all the joints, only critical regions or end-effector's location in the workspace are shown.}
    \label{fig:test_fetch}
\end{figure}
% \input{supplement.pdf}
  
%%%%%%%%%%%%%%%%%%%%%%%%%%%%%%%%%%%%%%%%%%%%%%%%%%%%%%%%%%%%%%%%%%%%%%%%

\end{document}